\documentclass[twoside,11pt]{article}
\usepackage[preprint]{jmlr2e}

\usepackage{multirow}
\usepackage{microtype}
\usepackage{graphicx}
\usepackage{subfigure}
\usepackage{booktabs} 
\usepackage{adjustbox}
\usepackage{amsmath}
\usepackage{amssymb} 
\usepackage{comment}
\usepackage{enumitem}
\newcommand{\propnumber}{} 

\newtheorem{rem}{Remark}

\usepackage{hyperref}

\usepackage{color,soul}

\newcommand{\squishlist}{
 \begin{list}{$\bullet$}
  { \setlength{\itemsep}{0pt}
     \setlength{\parsep}{1pt}
     \setlength{\topsep}{1pt}
     \setlength{\partopsep}{0pt}
     \setlength{\leftmargin}{1.5em}
     \setlength{\labelwidth}{1em}
     \setlength{\labelsep}{0.5em} } }
\newcommand{\squishend}{
  \end{list}  }

\newcommand{{\phivmat}}{\boldsymbol{{\phi}}}
\newcommand{{\phiv}}{\boldsymbol{{\phi}}}
\newcommand{\ba}[1]{\begin{align}#1\end{align}}

\usepackage{lastpage}

\ShortHeadings{Hyperparameter Tuning Through Pessimistic Bilevel Optimization}{Apaydin Ustun, Xu, Zeng and Qian}

\firstpageno{1}

\begin{document}

\title{Hyperparameter Tuning Through Pessimistic Bilevel Optimization}
\author{\name Meltem Apaydin Ustun \email mapaydin@nku.edu.tr \\
       \addr{Department of Electrical and Electronics Engineering, \\
       Tekirdag Namik Kemal University, Corlu, Tekirdag, Turkiye}  
       \AND
       \name Liang Xu \email lix21@pitt.edu \\
        \addr Department of Industrial Engineering\\
  University of Pittsburgh, Pittsburgh, PA 15261, USA 
  \AND
       \name Bo Zeng \email bzeng@pitt.edu \\
       \addr Department of Industrial Engineering\\
  University of Pittsburgh, Pittsburgh, PA 15261, USA 
       \AND
       \name Xiaoning Qian \email xqian@ece.tamu.edu \\
      \addr Department of Electrical and Computer Engineering\\
  Texas A\&M University, College Station, TX 77843, USA
       }

\editor{My editor}

\maketitle

\begin{abstract}%
Automated hyperparameter search in machine learning, especially for deep learning models, is typically formulated as a bilevel optimization problem, with hyperparameter values determined by the upper level and the model learning achieved by the lower-level problem. Most of the existing bilevel optimization solutions either assume the uniqueness of the optimal training model given hyperparameters or adopt an optimistic view when the non-uniqueness issue emerges. Potential model uncertainty may arise when training complex models with limited data, especially when the uniqueness assumption is violated. Thus, the suitability of the optimistic view underlying current bilevel hyperparameter optimization solutions is questionable. In this paper, we propose pessimistic bilevel hyperparameter optimization to assure appropriate outer-level hyperparameters to better generalize the inner-level learned models, by explicitly incorporating potential uncertainty of the inner-level solution set. To solve the resulting computationally challenging pessimistic bilevel optimization problem, we develop a novel relaxation-based approximation method. It derives pessimistic solutions with more robust prediction models. In our empirical studies of automated hyperparameter search for binary linear classifiers, pessimistic solutions have demonstrated better prediction performances than optimistic counterparts when we have limited training data or perturbed testing data, showing the necessity of considering pessimistic solutions besides existing optimistic ones.
\end{abstract}


\section{Introduction}
\label{intro}

Machine learning models often involve hyperparameters in training. Hyperparameter tuning is crucial for obtaining generalizable models but mostly done in heuristic ways until the recent emergent necessity of automated machine learning (AutoML) when training deep neural networks with a significantly large number of hyperparameters. In supervised learning, we seek a prediction model that performs well on unseen data, whose performance highly depends on the hyperparameters used for training. Note that, for given hyperparameters, we search for good models with respect to a chosen loss function based on training data. Hence, to achieve a desired performance on test data, hyperparameters should be selected carefully to constrain the search space of models in training. Such a testing-training hierarchical structure motivates to formulate the aforementioned problem as a bilevel program to support an automated process of hyperparameter tuning.  Specifically,  the outer level deals with the hyperparameter search while the inner level formulates the model training problem. 

For a bilevel optimization problem, there are two decision makers (DMs) interacting with each other. The outer-level DM has control over her decision variables (hyperparameters in our setting) $\textbf{x}\in \mathbb{R}^n$; and following her decision, the inner-level DM computes and provides the corresponding decision (model coefficients for example), $\textbf{y}\in \mathbb{R}^m$. Hence, the solution strategy of the inner level, $\textbf{y}$, is parameterized by the outer-level DM's decision $\textbf{x}$. Let $F(\textbf{x},\textbf{y})$ and $f(\textbf{x},\textbf{y})$ be the objective functions of the outer- and inner-level DMs, respectively. $G(\textbf{x})$ and $g(\textbf{x},\textbf{y})$ denote constraint functions of the outer- and inner-level DMs. For given outer-level decision variables $\textbf{x}$, we can formulate the inner-level problem as:
\ba{
\underset{\textbf{y}\in \mathbb{R}^m}{\text{min.}}~ \lbrace f(\textbf{x},\textbf{y})~:~g(\textbf{x},\textbf{y})\leq 0 \rbrace. \label{LLP}
}
Let $\Psi(\textbf{x})$ denote the solution set of the inner-level DM's problem in \eqref{LLP} for fixed $\textbf{x}$, i.e., $\text{argmin.}\lbrace f(\textbf{x},\textbf{y}):\textbf{y}\in\mathbb{R}^m, g(\textbf{x},\textbf{y})\leq 0\rbrace$. Clearly,  $\Psi$ is a point-to-set mapping from $\mathbb{R}^n$ to the power set of $\mathbb{R}^m$. Next, we express the outer-level problem:
\ba{
        \underset{\textbf{x}\in \mathbb{R}^n}{\text{``min."}}~ \lbrace F(\textbf{x},\textbf{y})~:~G(\textbf{x})\leq 0 , ~ \textbf{y} \in \Psi(\textbf{x})\rbrace. \label{ULP}
}        
Adopting the notations from \citet{dempe2002foundations}, we introduce the quotation marks to emphasize the uncertainty arisen from the solution set $\Psi(\textbf{x})$ when it has multiple elements for some or all $\textbf{x}$. In such situations, the outer-level DM is uncertain about the realization of the inner-level DM's decision. In general, there are two ways to overcome this issue: formulating the bilevel problem in 1) an optimistic, also known as strong formulation, or 2) a pessimistic view, also known as weak formulation. In the optimistic reformulation (i.e., \eqref{ULP} without the quotation marks) , the outer-level DM assumes the inner-level DM will be cooperative by providing a solution from his solution set in favor of outer-level DM's interest. For the pessimistic formulation, outer-level DM protects herself against unwanted decisions in the inner-level DM's solution set, which is the following three-level \textbf{PBL} model: 
\begin{equation}
    \begin{aligned}
        \text{\textbf{PBL:}}~~ \underset{\textbf{x}}{\text{min.}}~&\underset{\textbf{y}\in \Psi(\textbf{x})}{\text{max.}} ~ F(\textbf{x},\textbf{y})\\
        \text{s.t.} ~ & G(\textbf{x})\leq 0.
    \label{PBL}
    \end{aligned}
\end{equation}

Clearly, if the solution set $\Psi(\textbf{x})$ is guaranteed to be a singleton for any $\textbf{x}$, e.g., $f(\textbf{x},\textbf{y})$ is strictly convex in $\textbf{y}$ for any $\textbf{x}$, the optimistic and the pessimistic formulations coincide. Nevertheless, that guarantee typically does not hold. Non-uniqueness actually happens often when the inner-level problem is linear or non-convex. It also occurs when the inner-level DM solves his optimization problem approximately. For these cases, \textbf{PBL} helps the outer-level DM to bound the unwanted damage created by the unexpected realization of the inner-level DM's decision, and derives more robust solutions to against any possible realization of the inner-level solution. 

\subsection{An Empirical Study: Why a Pessimistic View? }
\label{motivEx}
\begin{figure}[h!]
\centering \vspace{-2mm}    
    \includegraphics[width=0.6\columnwidth]{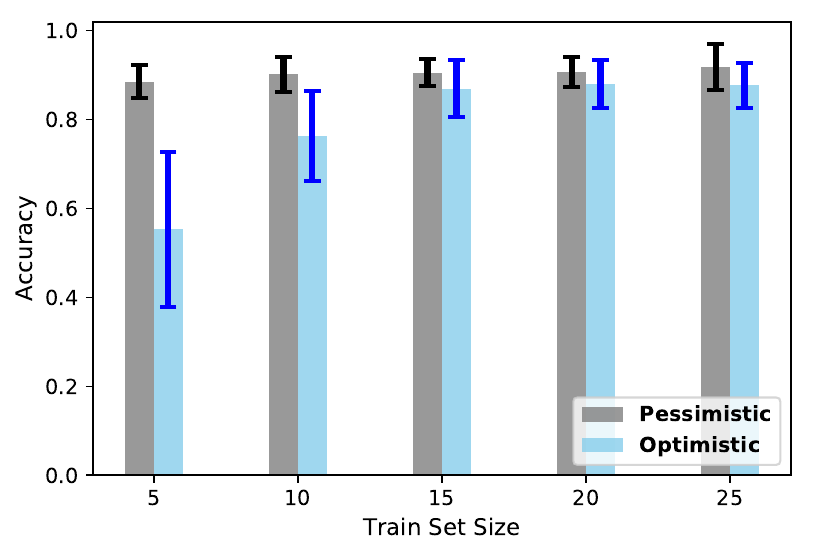}\vspace{-5mm}
\caption{Comparison of pessimistic and optimistic solutions by the average testing accuracy of ten runs with random splits between training and validation sets.}\vspace{-3mm}
\label{fig:motivEx}
\end{figure}

Hyperparameter tuning by bilevel optimization can be challenging when we only have limited amount of training data since the learned models can be highly uncertain and prone to overfitting. One critical question in this ``small data'' regime is: how shall we partition the limited data into training and validation sets to determine hyperparameters that lead to a generalizable model? We argue that through pessimistic bilevel optimization, we derive more robust models by optimizing hyperparameters to bound the damage due to the worst possible model performing on the validation set. Thanks to the inherent and strong interactions between the outer- and inner-level DMs, we expect that the pessimistic bilevel hyperparameter tuning will be \emph{less} sensitive to different splits between training and validation sets of limited data, compared to the optimistic solutions. 

To validate our discussion, we start with an empirical study of binary classification given data of small sample size. More details on the experimental setup are given in Section~\ref{smallData}. 
Figure \ref{fig:motivEx} presents the performances of the evaluated optimistic and pessimistic models 
for various training set sizes at a fixed validation set size of $10$. 
It is clear that the performance of the models obtained by the pessimistic solution is significantly better than those of the optimistic ones in testing accuracy. This difference is larger especially when we have a very limited training data. When the training set size is $5$, the accuracy of the evaluated optimistic models is almost no better than a random guess. Also, the variability of the pessimistic solutions' performances is generally smaller than that of the optimistic ones across runs.
Moreover, pessimistic models demonstrate a much more stable trend with the average accuracy not affected much by the size of the training set. 
These results motivate us to consider the pessimistic view to bilevel hyperparameter tuning for more robust solutions in uncertain scenarios of the inner-level training problem.


\subsection{Contributions}



Pessimistic bilevel optimization is challenging in terms of both computation and theoretical analysis \citep{dempe2014necessary}. In this paper, we first formulate a three-level \textbf{PBL} program for hyperparameter tuning to address the issues when the commonly assumed uniqueness is violated or when the inner-level model training is not solved to global optimal, which occurs frequently when fitting large models with limited training data. Moreover, by explicitly incorporating inner-level model uncertainty, the resulting \textbf{PBL} formulation may lead to more robust models when test data is different from training data. To solve this computationally challenging \textbf{PBL} program~\citep{wiesemann2013pessimistic, dempe2014necessary, zeng2020practical}, we employ a relaxation-based solution strategy \citep{zeng2020practical}. 
To the best of our knowledge, our study is the first to examine the suitability of the existing optimistic hyperparameter tuning and investigate the applicability of pessimistic bilevel optimization in machine learning. Our main contributions are:
\squishlist
    \item We investigate the suitability of the optimistic and pessimistic bilevel hyperparameter tuning solutions through empirical studies under different settings. To avoid complications due to suboptimality when considering complex models with no global optimality guarantee, we focus on binary linear classifiers. We show that optimistic solutions are indeed prone to overfitting with limited training data or shifted test data while pessimistic solutions are more robust, achieving higher test accuracy in these scenarios. 
    \item We for the first time formulate a three-level \textbf{PBL} program for hyperparameter tuning with a relaxation-based solution method. Using customized approximation techniques for binary classifiers, we convert the sophisticated non-convex and multilevel optimization problem into a computationally feasible formulation that can be solved to global optimality. 
    \item By introducing an $\varepsilon$-approximation of the inner-level solution set, we explicitly incorporate potential inner-level model uncertainty so that the learned model and hyperparameters are more robust with shifted test data. Such a capability is desired for AutoML in adversarial and meta-learning applications.  
    \item Under the transfer learning framework, we have tested our automated hyperparameter tuning with deep neural networks, where mapping to a feature space is achieved by using a convolutional neural network~(CNN) architecture and then optimistic and pessimistic models are optimized on the same mapped feature space. Our experimental results clearly show that the model obtained by our pessimistic bilevel optimization outperforms the model obtained by the optimistic one, again demonstrating the necessity of pessimistic AutoML solution strategies.    
\squishend

\section{Related Work}

A considerable amount of literature has been published on the interplay of hyperparameter tuning and bilevel optimization. \citet{bennett2006model} and  \citet{kunapuli2008bilevel} solved the model selection problem in support vector machines (SVM) by a bilevel optimization program. The authors utilized a K-fold cross-validation loss function as the outer-level objective for hyperparameter tuning with SVM as the inner-level training problem. The bilevel optimization formulation was converted into a single-level mathematical program with equilibrium constraints (MPEC) and solved by a nonlinear programming solver. 
More recently, bilevel formulations to optimize a huge number of hyperparameters have received considerable attention led by gradient-based methods specifically designed for that purpose~\citep{domke2012generic, maclaurin2015gradient, franceschi2017forward}. In these techniques, the general assumption is that the inner-level problem has a unique solution for any given outer-level hyperparameters. In the pioneering paper by \citet{domke2012generic}, based on the uniqueness assumption for the inner-level problem, a truncated bilevel optimization technique is introduced when the inner level is solved inexactly. Similarly, \citet{franceschi2018bilevel}  used gradient-based techniques to solve the bilevel hyperparameter optimization as well as the meta-learning problems, with the convergence guarantees of their iterative approach for the strongly-convex inner-level objective functions (so that the uniqueness assumption held). In these techniques, the general assumption is that the inner-level problem has a unique solution for any given outer-level hyperparameters~\citep{franceschi2017forward, maclaurin2015gradient, pedregosa2016hyperparameter, jenni2018deep, liu2021value}, or solutions converge to the optimistic one under multiple inner-level solutions \citep{mehra2019penalty}.
When the uniqueness assumption is not ensured, \citet{liu2020generic} proposed an effective strategy to address the non-uniqueness issue in the context of optimistic bilevel optimization. They designed a gradient-based approach for optimistic bilevel optimization that aggregates the outer-level objective to the inner-level objective function to guide the search direction in favor of the outer-level DM. Note that they assumed the strong-convexity of the outer-level objective function in their theoretical analysis. \citet{gao2022value} also focused on the optimistic view to the hyperparameter tuning under multiple inner-level solutions by exploiting the structure of its bilevel formulation.

Although the application of bilevel optimization and especially its pessimistic formulation is rather recent in machine learning, these methods have a long history in the optimization literature \citep{dempe2002foundations}. Much of the ongoing research on pessimistic bilevel optimization pays particular attention to linear problems, in which both the outer- and inner-level problems include linear functions in their objectives and constraints. One approach to attack such problems is to adopt penalty-based formulations. \citet{aboussoror2005weak} employed ideas from the duality theorem of the linear programming (LP) problem. They integrated a penalty term on the duality gap between primal and dual problems of the inner-level problem, and the corresponding single-level problem is solved by a sequence of bilinear programming problems. \citet{zheng2016solution} proposed an algorithm based on $K$-th best method for linear pessimistic problems. These penalty and enumeration based methods have been limited to linear pessimistic bilevel problems. 

In the works of \citet{loridan1996weak} and \citet{wiesemann2013pessimistic}, the authors addressed pessimistic bilevel problems not limited to linear ones; however, they assumed that the feasible region of the inner-level DM's problem is independent of the choice of the outer-level DM. 
The outer-level objective function is added to the inner-level problem to guide the search in an opposite direction to the leader's objective in the paper by \citet{loridan1996weak}, and the pessimistic bilevel model is approximated by solving a sequence of optimistic bilevel problems. \citet{wiesemann2013pessimistic} presented a significant analysis and discussion on the subject. In their work, the authors solved  an $\varepsilon$-approximation of the problem iteratively. One drawback of the algorithm is that it requires a global solution of sub-problems at every iteration. \citet{bergounioux2006regularization} proposed an algorithm using regularization to address the pessimistic bilevel problem; in which a penalized problem was solved approximately. A recent solution framework has been developed by \citet{zeng2020practical}, which, through a tight relaxation strategy,  enables us to make use of any existing methodology for optimistic bilevel problems to compute the pessimistic one. We employ this strategy in this paper as a basis to support our development of solution methods. More in-depth reviews regarding pessimistic bilevel optimization can be found in \citet{liu2018pessimistic} and \citet{dempe2018bilevel}. 

\section{Bilevel Hyperparameter Tuning}

In a prediction problem, we call $P_{x,y\sim D}[f(x)\neq y]$ as the true risk or generalization error of a prediction rule $f$ for a joint probability distribution $D$ over $X\times Y$ with inputs $x\in X$ and output response $y\in Y$. What we would like to find is a predictor $f$ that minimizes this error. In reality, we do not know the underlying data generating distribution $D$ or the prediction model space of $f$. For a general supervised learning problem, this true risk can also be defined as the expected loss by $f$ as $E_{x,y\sim D}[L(f(x), y)]$, where the loss function, $L$, describes the potential penalty to predict $f(x)$ when it deviates from the true label $y$. Also note that when the loss function is a 0-1 loss for a classification problem, defined as
\begin{equation}
    L_{0-1}(f(x), y)=
    \begin{cases}
      0 & \text{if $f(x)=y$}\\
      1 & \text{if $f(x)\neq y$},
    \end{cases} 
\end{equation}
the expected loss value and the generalization error coincide. From now on, we use the general risk function as $E_{x,y\sim D}[L(f(x), y)]$.

In practice, when we only have access to sample data, empirical risk is taken into account, defined as
\begin{equation}
\label{erm-obj}
    L_S(f) = \frac{1}{m}\sum_{i=1}^m L(f(x_i), y_i)
\end{equation}
over a given sample $S=\lbrace (x_1,y_1), (x_2, y_2),\dots , (x_m, y_m)\rbrace$. The problem then becomes to minimize this empirical risk measure (ERM). In practice, ERM is optimized over a model space, and this learning process is controlled by tuning so-called hyperparameters. For example, a regularization term is often added to the ERM objective in~\eqref{erm-obj} to control the complexity of the model, which constrains the search space of the model parameters. In many practical cases, such constraints require the hyperparameters of the model or algorithm to be tuned carefully to obtain a generalizable model. In most cases, the data set is divided into two parts: training and validation. The model's parameters are determined by an ERM procedure on the training set, and hyperparameters are then determined by minimizing a loss function on the validation set. As hyperparameter tuning is crucial for obtaining generalizable test performance, this problem can be automated by solving its bilevel formulation as
\ba{\small
\label{general:bilevel}
\underset{\lambda\in\Lambda}{\text{min.}}~&L_{S_{\text{val}}}(f(\theta)) \\
& \theta \in \Psi(\lambda), \nonumber
}
where
\ba{\small
\label{general-BL-inner}
\Psi(\lambda) = \underset{\theta\in\Theta}{\text{argmin.}}\lbrace L_{S_{\text{train}}} (f(\theta)):~g(\theta,\lambda)\leq 0 \rbrace.
}
Here, $\theta \in \Theta$ denotes the model parameters, and $\lambda \in \Lambda$ corresponds to hyperparameters of the model training. The constraint in \eqref{general-BL-inner} restricts the search space of the model parameters that is controlled by $\lambda$, and this dependence can be introduced in a general form as $g(\theta, \lambda)$. Hence, the outer-level objective is to minimize the loss on validation data by optimizing $\lambda$, and the model parameters belong to the parameterized optimal solution set (i.e., $\Psi(\lambda)$) of the inner-level training problem. When this set is not a singleton, there are two approaches to reformulate the problem: optimistic and pessimistic reformulations.

In the subsequent sections, we briefly describe the popular optimistic reformulation, and then propose the pessimistic reformulation and the solution strategy to solve the resulting computationally challenging three-level problem. 

\subsection{Optimistic Bilevel Hyperparameter Tuning}
The optimistic reformulation of the problem \eqref{general:bilevel}-\eqref{general-BL-inner} is
\ba{\small
\label{general-O:bilevel}
P_o^*=\underset{\lambda\in\Lambda, \theta\in\Psi(\lambda)}{\text{min.}}&L_{S_{\text{val}}}(f(\theta)).
} 
When the inner-level problem in \eqref{general-BL-inner} is not strictly convex, e.g. a linear program (LP), it is possible that for some given outer-level hyperparameters, the inner-level optimal solution set may have multiple elements. The issue becomes apparent when there's limited data available for division between training and validation sets. Under these circumstances, even with optimized hyperparameters obtained by solving the problem in \eqref{general-O:bilevel}, the resulting inner-level model may struggle to generalize effectively when applied to unseen test data, especially when test data is generated with a shifted data generation process. 
When training complex models with non-convex objective functions, this can be even more problematic as we may only be able to derive suboptimal solutions for the inner-level problem. Hence, it is more desired to adopt a pessimistic view to perform bilevel hyperparameter tuning. By so, 
we seek to optimize the hyperparameters so that the derived models still perform well with respect to an out-of-sample validation loss under the worst-possible model that achieves similar training performance. 

\subsection{Pessimistic Bilevel Hyperparameter Tuning}
\label{Section:PBL-general}
The pessimistic reformulation of the problem \eqref{general:bilevel}-\eqref{general-BL-inner} is
\ba{\small
\label{general-P:bilevel}
P_p^*=\underset{\lambda\in\Lambda}{\text{min.}} \underset{\theta\in\Psi(\lambda)}{\text{max.}} & L_{S_{\text{val}}}(f(\theta)).
}
Let $(\lambda^*_o,\theta^*_o)$ denote an optimal solution to the outer-level problem of optimistic formulation~\eqref{general-O:bilevel}, and $P_p(\lambda,\cdot)$ the optimal value of \eqref{general-P:bilevel} for given  $\lambda$. The next result follows straightforwardly. 
\begin{proposition}
\label{prop_pb_ob}
We have $P^*_o\leq P^*_p\leq P_p(\lambda^*_o,\cdot)$, and the equalities hold if $\Psi(\lambda)$ is a singleton for all $\lambda$.
\end{proposition}

Proposition 1 indicates that the optimal value of the pessimistic reformulation is upper-bounded by the equation~\eqref{general-P:bilevel} evaluated at the optimum solution $\lambda_o^*$ of the optimistic formulation. We note that the three-level pessimistic bilevel formulation in \eqref{general-P:bilevel} for hyperparameter tuning captures richer information on the interactions between outer- and inner-level DMs. Specifically, when there is a limited amount of training data, the inner-level problem may have multiple solutions, of which not all are generalizable well to test data. When hyperparameters are optimized in a way such that the worst-case model still performs well on validation data, as shown in the pessimistic model obtained by \eqref{general-P:bilevel}, they should achieve an improved performance on test data as demonstrated in our experiments. Furthermore, if test data is corrupted by an attacker, this pessimistic formulation yields more robust models by taking a worst-case view.


To solve the computationally challenging three-level problem in \eqref{general-P:bilevel}, we next adopt a reformulation strategy presented in~\citet{zeng2020practical}, which reduces its computation to solving a bilevel optimization problem by introducing $\bar\theta$, the replica of decision variables of $\theta$, into the outer-level problem.

\begin{proposition}
The following bilevel optimization problem is a tight relaxation of the three-level formulation in \eqref{general-P:bilevel}, and its optimal solution $\lambda^*$ is also  optimal to \eqref{general-P:bilevel}.
\ba{\small
\label{general-P-relax:bilevel}
P^*_p=\underset{\begin{subarray}{c}
  \lambda\in\Lambda, \bar\theta\in\Theta, \\
  g(\bar\theta,\lambda) \leq 0
  \end{subarray}}{\text{min.}} & \underset{\theta\in\Theta}{\text{max.}} \Large\{ L_{S_{\text{val}}}(f(\theta)): g(\theta,\lambda) \leq 0,L_{S_{\text{train}}} (f(\theta)) \leq L_{S_{\text{train}}} (f(\bar\theta))
\Large\}.}
\end{proposition}

According to \cite{zeng2020practical}, the last constraint in \eqref{general-P-relax:bilevel} directly captures the interaction between the outer- and inner-level problems, which guarantees that optimal $\lambda^*$ to \eqref{general-P-relax:bilevel} is also optimal to $\eqref{general-P:bilevel}$. Clearly, this reformulation greatly reduces the complexity of computing \eqref{general-P:bilevel} and converts a pessimistic formulation into a $\min$-$\max$ problem. 
However, we note that the problem in \eqref{general-P-relax:bilevel} still poses non-trivial challenges. First, when the loss function is chosen to be a 0-1 loss, which is a common practice in grid search strategies to tune hyperparameters, the inner-level maximization problem \eqref{general-P-relax:bilevel} becomes non-concave. If the 0-1 loss is also chosen as a training loss, this increases the difficulty even further. Even if we use convex surrogates of the 0-1 loss in both objectives of inner- and outer-level problems, the inner maximization in \eqref{general-P-relax:bilevel} still becomes non-concave. 


In fact, a binary classification problem in (10) can be converted to a $\min$-$\min$ problem, where both the outer- and inner-level objective functions are ERM objectives. For the sake of clarity in notation, we integrate the last relaxation constraint,  $L_{S_{\text{train}}} (f(\theta)) \leq L_{S_{\text{train}}} (f(\bar\theta))$, into the constraint $g(\lambda, \theta) \leq 0$ and denote the resulting constraint by $\bar g(\lambda, \theta, \bar\theta) \leq 0$. Then we have the following result.

\begin{theorem}\label{thm:flipping}
For a binary classification problem with labels $y \in \lbrace -1, +1\rbrace$, the tight relaxation of the pessimistic bilevel hyperparameter optimization problem, which is stated as
\ba{\small
\label{eq:minmax-RP-binary}
\underset{\begin{subarray}{c}
  \lambda\in\Lambda, \bar\theta\in\Theta, \\
  g(\lambda,\bar\theta) \leq 0
  \end{subarray}}{\text{min.}}\lbrace \underset{\theta\in\Theta}{\text{max.}}~\sum_{i\in V} I(y_i f_\theta(x_i)<0):~\bar g(\lambda, \theta, \bar\theta) \leq 0\rbrace,
}
is equivalent to the following optimization problem under the 0-1 loss function:
\ba{\small
\label{eq:minmin-RP-binary}
&\underset{\begin{subarray}{c}
  \lambda\in\Lambda, \bar\theta\in\Theta, \\
  g(\lambda,\bar\theta) \leq 0
  \end{subarray}}{\text{min.}} \sum_{i\in V}I(y_i f_{\theta^*}(x_i)<0)\\
&\text{where}~\theta^*\in \underset{\theta\in\Theta}{\text{argmin.}} \lbrace \sum_{i\in V} I(\bar{y_i} f_\theta(x_i)<0):~\bar g(\lambda, \theta, \bar\theta) \leq 0\rbrace. \nonumber
}
Here, the predictor model $f$ is parameterized by $\theta$, $\lambda$ denotes the hyperparameters of the problem, $\bar\theta$ is the replicated inner-level decision variables, and $\bar{y_i}$ denotes the flipped label for $i$-th data point in the validation set. 

\begin{proof}
Note that $\sum_{i\in V}I(y_i f_\theta(x_i) < 0) = |V| - \sum_{i\in V}I(\bar{y}_if_\theta(x_i) < 0)$ where $I(\cdot)$ denotes the indicator function, and $\bar{y}_i$ is the flipped label ($\bar{y}=-1$ if $y=1$, and $\bar{y}=1$ if $y=-1$). Then,
\begin{align}
    &\underset{\begin{subarray}{c}
  \lambda\in\Lambda, \bar\theta\in\Theta, \\
  g(\lambda,\bar\theta) \leq 0
  \end{subarray}}{\text{min.}}\big \lbrace \underset{\theta\in\Theta}{\text{max.}}~ \sum_{i\in  V} I(y_if_\theta(x_i) < 0)~:~ \bar g(\lambda, \theta, \bar\theta) \leq 0 \big \rbrace \nonumber\\
    = & \underset{\begin{subarray}{c}
  \lambda\in\Lambda, \bar\theta\in\Theta, \\
  g(\lambda,\bar\theta) \leq 0
  \end{subarray}}{\text{min.}}\big \lbrace \underset{\theta\in\Theta}{\text{max.}}~ |V|-\sum_{i\in V} I(\bar{y}_if_\theta(x_i) < 0)~:~ \bar g(\lambda, \theta, \bar\theta) \leq 0\big \rbrace \nonumber\\
    = & \underset{\begin{subarray}{c}
  \lambda\in\Lambda, \bar\theta\in\Theta, \\
  g(\lambda,\bar\theta) \leq 0
  \end{subarray}}{\text{min.}}~ |V|- \underset{\theta\in\Theta}{\text{min.}}~ \big \lbrace\sum_{i\in V} I(\bar{y}_if_\theta(x_i) < 0)~:~ \bar g(\lambda, \theta, \bar\theta) \leq 0 \big \rbrace, \nonumber
\end{align}
where the last equality follows the fact that $\text{max}~f(\theta)\equiv -\text{min}\{-f(\theta)\}$ and that $|V|$ is a constant term. This is a bilevel problem with the outer-level problem 
\begin{equation}
    \underset{\begin{subarray}{c}
  \lambda\in\Lambda, \bar\theta\in\Theta, \\
  g(\lambda,\bar\theta) \leq 0
  \end{subarray}}{\text{min.}}~|V|-\sum_{i\in V} I(\bar{y}_i f_{\theta^*}(x_i) < 0) = \underset{\begin{subarray}{c}
  \lambda\in\Lambda, \bar\theta\in\Theta, \\
  g(\lambda,\bar\theta) \leq 0
  \end{subarray}}{\text{min.}}~\sum_{i\in V} I(y_if_{\theta^*}(x_i) < 0),
    \label{myeq:ULflip}
\end{equation}
where 
\begin{equation}
    \theta^*\in \underset{\theta}{\text{argmin.}}\lbrace ~\sum_{i\in V} I(\bar{y}_if_{\theta}(x_i) < 0)~:~\bar g(\lambda, \theta, \bar\theta) \leq 0 \rbrace.
    \label{myeq:LLflip}
\end{equation}
\end{proof}
\end{theorem}

\begin{rem}
We here note that Theorem~\ref{thm:flipping} is an important result in two aspects. First, it yields a simple but equivalent reformulation of the $\min$-$\max$ problem \eqref{general-P-relax:bilevel} into a $\min$-$\min$ problem. Furthermore, we can use surrogate convex losses in \eqref{eq:minmin-RP-binary}, a common practice in machine learning applications, to obtain computationally more friendly approximations. For example, assuming the inner-level problem is a convex optimization problem, then, by its optimality conditions, \eqref{eq:minmin-RP-binary} can be converted into a single-level problem. 
\end{rem}

To have a fair comparison with the optimistic model and to be able to solve the problem to optimality, we focus on such loss functions. Specifically, under the hinge loss, $L$, with a linear model, $f$, the inner-level problem can be cast as an LP problem. Furthermore, the hinge loss, $\text{max}(0,1 - yf_\theta(x))$, can be seen as a convex approximation of $I(yf_\theta(x) < 0)$, i.e., an upper bound on the 0-1 loss. 
However, flipping all the labels of the data in the validation set to convert the maximization into minimization problem creates a mismatch because of the margin introduced when using the hinge loss. Hence, instead of the 0-1 loss, we focus on a modified version, $I(yf_\theta(x) < 1)$. With $I(yf_\theta(x) < 1)$ used as a loss function, we not only aim to minimize the number of misclassified points but also the number of the ones falling into the margin defined by $\theta$. And, the problem in~\eqref{eq:minmax-RP-binary} is upper-bounded by 
\ba{\small
\label{eq:minmax-RP-binaryHinge}
\underset{\begin{subarray}{c}
  \lambda\in\Lambda, \bar\theta\in\Theta, \\
  g(\lambda,\bar\theta) \leq 0
  \end{subarray}}{\text{min.}}\lbrace \underset{\theta\in\Theta}{\text{max.}}~\sum_{i\in V} I(y_i f_\theta(x_i)<1):~\bar g(\lambda, \theta, \bar\theta) \leq 0\rbrace.
}
Now, we can flip the labels of not all but some points in the validation set to transform the problem into a $\min$-$\min$ problem. How this flipping can be performed is provided in the following Corollary~\ref{corol:flip}. First, we partition the validation set into three disjoint subsets, $|V_1|+|V_2|+|V_3| = |V|$. Let $V_1$ denote the set of indices of points that fall inside the margin determined by $\theta$, $V_2$ the set of indices of out-of-margin points classified correctly by $f$ parameterized by $\theta$, and similarly $V_3$ the set of indices of out-of-margin points that are misclassified by $f$:
\begin{align*}
&V_1 = \lbrace i \in V : | f_\theta(x_i) | < 1 \rbrace, \\
&V_2 = \lbrace i \in V : | f_\theta(x_i) | \geq 1 ~\text{and}~ y_i f_\theta(x_i) \geq 1 \rbrace, \\
&V_3 = \lbrace i \in V : | f_\theta(x_i)| \geq 1 ~\text{and}~ y_i f_\theta(x_i) \leq -1 \rbrace. 
\end{align*}

\begin{corollary}
\label{corol:flip}
$V_1, V_2$ and $V_3$ are disjoint sets, and the objective function value of \eqref{eq:minmax-RP-binaryHinge} for a given $\theta$ is $\sum_{i\in V} I(y_i f_\theta(x_i)<1) = |V_1| + |V_3| = \sum_{i\in V}I(i\in V_1 \lor i\in V_3)$. Hence, we can rewrite the problem in \eqref{eq:minmax-RP-binaryHinge} as:
\ba{
\label{eq:minmax-RP-binary-PF}
&\underset{\begin{subarray}{c}
  \lambda\in\Lambda, \bar\theta\in\Theta, \\
  g(\lambda,\bar\theta) \leq 0
  \end{subarray}}{\text{min.}}\lbrace \underset{\theta\in\Theta}{\text{max.}}~\sum_{i\in V} I(i\in V_1 \lor i\in V_3):~\bar g(\lambda, \theta, \bar\theta) \leq 0\rbrace\\
&  = \underset{\begin{subarray}{c}
  \lambda\in\Lambda, \bar\theta\in\Theta, \\
  g(\lambda,\bar\theta) \leq 0
  \end{subarray}}{\text{min.}} \lbrace |V|-\underset{\theta\in\Theta}{\text{min.}} \sum_{i\in V} I((i\in V_1)^c \land (i\in V_3)^c):~\bar g(\lambda, \theta, \bar\theta) \leq 0\rbrace,\\
& =  \underset{\begin{subarray}{c}
  \lambda\in\Lambda, \bar\theta\in\Theta, \\
  g(\lambda,\bar\theta) \leq 0
  \end{subarray}}{\text{min.}} \lbrace |V|-\underset{\theta\in\Theta}{\text{min.}} \sum_{i\in V} I(i\in V_2):~\bar g(\lambda, \theta, \bar\theta) \leq 0\rbrace.\label{min-min:PF-binary}
}

For $i\in V$, $I(i\in V_2)$ is the same as $ I(\bar y_i f_\theta(x_i)\leq -1 )$ for a given $\theta$, and $\bar y$ denotes the flipped label. 
\end{corollary}
The inner level of \eqref{min-min:PF-binary} is actually minimizing the number of misclassified points outside the margin. Hence, in our implementation, we flip the labels of misclassified points by the optimistic hyperplane considering Proposition 1. Additionally, we also consider the points inside the margin determined based on the hyperplane obtained by the optimistic counterpart of the problem to have a more robust model. 

Let $V_f$ denote the set of indices of misclassified and marginal points. Intuitively, this pessimistic formulation for hyperparameter tuning can be viewed as creating adversarial examples by flipping the labels of points from $V_f$ and encouraging the model to have robust performance considering this worst-case scenario. We also note that these operations are on the validation set. 
Finally, if we further adopt the hinge loss with a linear model for $f$, we can use optimality conditions to convert the problem into a single-level formulation as the resulting inner-level minimization problem on $V_f$ with label flipping can be cast as an LP program. Furthermore, this partial flipping operation further reduces the computational burden as it significantly reduces the number of constraints of KKT (Karush-Kuhn-Tucker) conditions while converting the problem into a single level formulation. Details of the procedure and the resulting single-level optimization problem under the hinge loss are provided in Section~\ref{section:hinge}.

\subsection{Extension of the Pessimistic Bilevel Hyperparameter Optimization to Handle Model Uncertainty}

We further explicitly incorporate possible inner-level training model uncertainty by introducing an $\varepsilon$-approximation of the inner-level problem's solution set.  First, the inner-level training problem may not be solved to the global optimality. In this case, we can obtain an enlarged optimal solution set through the $\varepsilon$-approximation approach, which takes suboptimality into consideration. Second, the incorporation of $\varepsilon$ increases the flexibility of a learning model. Such flexibility is very crucial if a shifted or an attacked test data set is expected. Actually, rather than solely relying on the model based on the training data, it is anticipated that a better generalization capability can be achieved by incorporating a safety margin to hedge against deviations from the expected trained model, i.e., by taking its ``neighboring'' models into the account. Finally, even if the solution set of the inner-level problem is a singleton, such approximate pessimistic formulation would allow us to understand how sensitive the hyperparameters are under an enlarged optimal solution set of the inner-level trained models. By explicitly considering this sensitivity issue, the pessimistic hyperparameter tuning formulation is expected to yield more reliable models.

For the pessimistic hyperparameter optimization problem, suboptimality $\varepsilon$ can be chosen as an absolute or a relative value with respect to the inner-level training problem's loss. With the fact that $L_{S_\text{train}} (f(\bar\theta))\geq 0$  and using the relative rate, such an $\varepsilon$-approximation of the pessimistic problem can be written as follows:
\ba{\small
\label{general-P-relax-flip-eps:bilevel}
P_p^\varepsilon= \underset{\lambda\in\Lambda, \bar\theta\in\Theta}{\text{min.}}& 
L_{S_{\text{val}}}(f(\theta)) \\
& \text{s.t.}~ g(\lambda,\bar\theta) \leq 0. \nonumber \\
& \theta \in \underset{\theta\in\Theta}{\text{argmax.}}\lbrace L_{S_\text{val}}(f(\theta)):~ L_{S_{train}} (f(\theta)) \leq (1+\varepsilon)L_{S_\text{train}} (f(\bar\theta)),~g(\lambda,\theta) \leq 0 \rbrace.\nonumber
}

\begin{proposition}
 Function  $P_p^\varepsilon$ is increasing with respect to $\varepsilon$. Specifically, considering $0\leq \varepsilon_1\leq \varepsilon_2$, we have $$P^{\varepsilon_2}_p(\lambda^*_o,\cdot)\geq P_p^{\varepsilon_2}\geq P_p^{\varepsilon_1}\geq P_p^0\equiv P^*_p\geq P^*_o.$$
\end{proposition}
\begin{proof}
According to Proposition \ref{prop_pb_ob} and \eqref{general-P-relax-flip-eps:bilevel}, it is sufficient to prove that $P_p^{\varepsilon_2}\geq P_p^{\varepsilon_1}$. 

Let $P_p^\varepsilon(\lambda, \bar\theta)$ denote the optimal value of \eqref{general-P-relax-flip-eps:bilevel} subject to fixed $\varepsilon, \lambda$, and $\bar\theta$. Suppose that  $(\lambda_i, \bar\theta_i)$ is an optimal solution to the outer-level problem of \eqref{general-P-relax-flip-eps:bilevel} for $\varepsilon_i$, $i=1$ and 2 respectively. We have 
$$P_p^{\varepsilon_2} \equiv P^{\varepsilon_2}_p(\lambda_2, \bar\theta_2)\geq P^{\varepsilon_1}_p(\lambda_2, \bar\theta_2)\geq P_p^{\varepsilon_1}(\lambda_1, \bar\theta_1) \equiv P_p^{\varepsilon_1}.$$
The first inequality simply follows from the fact that the larger $\varepsilon$ the greater $\max_{\theta } L_{S_{val}}(f(\theta))$ for any given $(\lambda,\bar\theta)$. The second inequality is straightforward, as $(\lambda_1, \bar\theta_1)$ is an optimal solution with respect to $\varepsilon_1$. Moreover, as $\lambda_o^*$ might not be optimal to $ P_p^{\varepsilon_2}$, it follows that $P^{\varepsilon_2}_p(\lambda^*_o,\cdot)\geq P_p^{\varepsilon_2}.$ Hence, the expected result is obtained. 
\end{proof}
As Proposition 5 indicates, increasing the value of $\varepsilon$ results in a larger solution set for the inner-level problem (19). This means that we adopt a more conservative approach, which may degrade the performance on the validation set. However, having a larger optimal solution set for the candidate hyperplanes is beneficial when the test data comes from a shifted distribution that differs from the training/validation data. Although the exact shift amount is impossible to know in practice, it can be compensated by carefully choosing $\varepsilon$'s value. 

\section{Solution Methods to Pessimistic Bilevel Hyperparameter Tuning}\label{section:hinge}

In this section, we first present the optimistic view to the bilevel hyperparameter tuning problem under a specific convex loss function, namely hinge loss,  with a discussion of possible issues encountered under small-sample scenarios. Then, we present the new  pessimistic formulation and its solution method.

\subsection{Optimistic Reformulation}
In this paper, we focus on learning a linear SVM, where a box constraint on the model coefficients $w$ is added as a constraint:  
\ba{\small
        \underset{w\in\mathbb{R}^p,b,\xi}{\text{min.}}~ & \dfrac{1}{|T|} \sum_{i\in T} \xi_i \label{boxSVM}\\
        \text{s.t.}~& y_i(x_i^Tw-b) \geq 1-\xi_i, ~\forall i \in T \nonumber \\
        & \xi_i \geq 0 , ~\forall i \in T \nonumber \\
        & -\bar{w} \leq w \leq \bar{w}. \nonumber
}
Here $\bar{w}\in\mathbb{R}^p$ is a vector of hyperparameters that needs to be given \emph{a priori}, $T$ is the index set of training data, $w$ and $b$ are the parameters of the linear model, and $\xi$'s are the slack variables. This formulation has been studied by \citet{bennett2006model} and \citet{kunapuli2008bilevel}. Similar to the purpose of the box constraint in \eqref{boxSVM}, a hard constraint on $||w||_\infty$ corresponding to the $l_\infty$-norm regularization can be motivated for the robustness where it may provide defense against adversarial attacks in the dual $l_1$-norm. Careful determination of such bounds is therefore crucial to learn robust classifiers. 

Denote $V$ as the index set of the labeled validation data, $\lbrace (x_i ,y_i): i \in V \rbrace$, apart from the training data. The optimistic bilevel hyperparameter optimization problem is then formulated as:
\ba{\small
\label{prob:optim-hinge}
       \underset{\bar{w}, w,b}{\text{min.}}~ & \dfrac{1}{|V|}~\sum_{i\in V}[1-y_i(x_i^Tw-b)]_+ \\
        \text{s.t.}~ & \bar{w}^{\text{LB}} \leq \bar{w}\leq \bar{w}^{\text{UB}} \nonumber\\
        & (w,b) \in \underset{w',b',\xi}{\text{argmin.}}\lbrace \dfrac{1}{|T|}\sum_{i\in T}\xi_i: \xi_i \geq 1-y_i(x_i^Tw'-b'), \nonumber \\
        & \qquad \qquad \xi_i\geq0,~\forall i \in T, -\bar{w}\leq w'\leq \bar{w}\rbrace, \nonumber
}        
where $\bar{w}^{\text{LB}}$ and $\bar{w}^{\text{UB}}$ are the lower- and upper-bounds on the hyperparameters $\bar{w}$. In the outer-level objective function of \eqref{prob:optim-hinge}, we use the hinge loss, denoted as $[1-y(x^Tw - b)]_+ := \text{max}(0,1-y(x^Tw - b))$. 
As we are concerned with the performance comparison of the corresponding pessimistic bilevel formulation with respect to this optimistic formulation, we focus on the simple out-of-sample validation loss in \eqref{prob:optim-hinge}.


As the inner-level problem in \eqref{prob:optim-hinge} is an LP, it is possible that for some given outer-level hyperparameter values, the inner-level optimal solution set may have multiple elements. This problem would manifest, in particular under the small data regime when we have to split the limited data into training and validation sets for hyperparameter tuning. In this case, the derived inner-level model, even under optimized hyperparameters by solving \eqref{prob:optim-hinge}, may have a poor generalizability on testing data. 
Instead, we seek to optimize the hyperparameters so that the derived models still perform well with respect to an out-of-sample validation loss.



\subsection{Pessimistic View to Hyperparameter Tuning}

We now formulate the pessimistic view of the bilevel hyperparameter optimization problem as described in Section \ref{Section:PBL-general}:
%
\ba{\small
\underset{\bar{w}\in\mathbb{R}^p}{\text{min.}}~ & \underset{w\in\mathbb{R}^p,b}{\text{max.}}~ \dfrac{1}{|V|}~\sum_{i\in V}[1-y_i(x_i^Tw-b)]_+ \label{PBL-hyper}\\
        \text{s.t.}~& \bar{w}^{\text{LB}} \leq \bar{w}\leq \bar{w}^{\text{UB}} \nonumber \\
        & (w,b)\in \underset{w',b',\xi}{\text{argmin.}} \lbrace \dfrac{1}{|T|}\sum_{i\in T}\xi_i:\xi_i \geq 1-y_i(x_i^Tw'-b'), \nonumber \\
        &\qquad \qquad \xi_i\geq0,~\forall i \in T, -\bar{w}\leq w'\leq \bar{w}\rbrace. \nonumber
}

Now, a relaxation for the problem in \eqref{PBL-hyper} can be obtained by replicating the inner-level decision variables as well as the constraints to the outer level with further incorporation of the model uncertainty:
\begin{equation}
\label{RPBL-hyper}
    \begin{aligned}
    \underset{\bar{w},\hat{w},\hat{b},\hat{\xi}}{\text{min.}}~\underset{w,b,\xi}{\text{max.}}~&\dfrac{1}{|V|}\sum_{i\in V} [1-y_i(x_i^T w - b)]_+ \\
    \text{s.t.}~~ &\bar{w}^{LB} \leq \bar{w} \leq \bar{w}^{UB} \\
    & -\bar{w} \leq \hat{w} \leq \bar{w} \\
    & -\bar{w} \leq w \leq \bar{w} \\
    & \hat{\xi}_i \geq 1-y_i(x_i^T \hat{w} - \hat{b}), ~ \forall i \in T\\
    & \hat{\xi}_i \geq 0,~ \forall i \in  T \\
    & \xi_i \geq 1-y_i(x_i^T w - b), ~ \forall i \in T \\
    & \xi_i \geq 0,~ \forall i \in T \\
    & \dfrac{1}{|T|}\sum_{i \in ~ T} \xi_i \leq \dfrac{1+\varepsilon}{|T|}\sum_{i \in T} \hat{\xi}_i.
    \end{aligned}\vspace{-1mm}
\end{equation}
In this formulation, $(\hat{w},\hat{b},\hat{\xi})$ are the replicated inner-level decision variables.

As detailed in Theorem \ref{thm:flipping}, we can convert non-concave maximization problem at the inner level of \eqref{RPBL-hyper} into a classical ERM framework. This conversion was exact if we had used 0-1 loss at the outer level of \eqref{RPBL-hyper}. However, we approximate it by a surrogate convex loss function, namely the hinge loss function $\text{max}(0, 1-y(x^Tw-b))$, to solve the problem to global optimality.
This convex approximation creates a mismatch with the intuition behind the flipping operation in Theorem \ref{thm:flipping} because of the introduced margin with the hinge loss. Therefore, we specialize the flipping operation to handle this mismatch by flipping the samples corresponding to the indices in the set $V_f$ as given in Corollary~\ref{corol:flip}. 
This provides a more robust model as the separating hyperplane is more susceptible to the possible perturbations of those samples closer to the decision boundary. Indeed, the data with flipped labels acts as adversarial examples in the validation set. In this implementation, the margin needs to be determined in advance. Hence, we first solve the optimistic bilevel hyperparameter optimization problem, and then compute the margin, $2/||w||$, with respect to the decision hyperplane obtained by solving \eqref{prob:optim-hinge}. Then $V_f$ is determined based on the optimistic hyperplane and the corresponding margin. 


With these operations, we finally derive our approximate formulation to \eqref{RPBL-hyper}, which is adopted in our hyperparameter tuning implementation: 
\begingroup
\allowdisplaybreaks
\ba{\small
\underset{\bar{w},\hat{w},\hat{b},\hat{\xi}}{\text{min.}} ~ &\dfrac{1}{|V|}\sum_{i\in V} [1-y_i(x_i^T w - b)]_+ \label{A-RPBL-hyper}\\
    \text{s.t.}~~ &\bar{w}^{LB} \leq \bar{w} \leq \bar{w}^{UB} \nonumber \\
    & -\bar{w} \leq \hat{w} \leq \bar{w} \nonumber \\
    & \hat{\xi}_i \geq 1-y_i(x_i^T \hat{w} - \hat{b}), ~ \forall i \in T \nonumber \\
    & \hat{\xi}_i \geq 0,~ \forall i \in T \nonumber \\
    & (w,b,v,\xi) \in \underset{v_i \geq 0,\forall i \in V_f,~\xi_i \geq 0, \forall i \in T}{\text{arg min.}}\big\lbrace \dfrac{1}{|V_f|}\sum_{i\in V_f} v_i \nonumber \\
    & \qquad  \qquad \qquad v_i \geq 1 - \bar{y}_i (x_i^T w- b),~ \forall i \in ~ V_f \nonumber \\
& \qquad  \qquad \qquad -\bar{w} \leq w \leq \bar{w} \nonumber \\
    & \qquad  \qquad \qquad \xi_i \geq 1-y_i(x_i^T w - b), ~ \forall i \in T \nonumber \\
 & \qquad  \qquad \qquad \dfrac{1}{|T|}\sum_{i \in ~ T} \xi_i \leq \dfrac{1+\varepsilon}{|T|}\sum_{i \in ~ T} \hat{\xi}_i\big\rbrace. \nonumber 
}
\endgroup
$V_f$ is the index set of the flipped points, $\bar{y}_i$ is the flipped label for the $i$-th example. 
As the optimization problem in \eqref{A-RPBL-hyper} is a standard bilevel linear optimization problem, we convert it to a single-level formulation through the KKT conditions as follows: 
\newcommand\numberthis{\addtocounter{equation}{1}\tag{\theequation}} 
\begingroup
\allowdisplaybreaks
   \begin{align*}
    \text{min.}~~ & \dfrac{1}{|V|} \sum_{i\in V} g_i \\
    & \text{\textit{Outer-level constraints}} \\
    \text{s.t.}~~ & g_i \geq 1-y_i(x_i^Tw-b), ~~ g_i\geq 0,~\forall i \in V \\
    & \bar{w}^{LB} \leq \bar{w} \leq \bar{w}^{UB}  \label{myeq:single-hyper} \\
    & -\bar{w} \leq \hat{w} \leq \bar{w} \\
    & \hat{\xi}_i \geq 1-y_i(x_i^T\hat{w}-\hat{b}), ~~ \hat{\xi}_i \geq 0, ~\forall i \in T \\
    & \text{\textit{Primal constraints}} \\
    & v_i \geq 1-\bar{y}_i(x_i^Tw-b), ~~ v_i\geq 0,~\forall i \in V_f \\
    & -\bar{w} \leq w \leq \bar{w} \\
    & \dfrac{1}{|T|} \sum_{i \in T} \xi_i \leq \dfrac{1+\varepsilon}{|T|} \sum_{i \in T} \hat{\xi}_i \\
    & \xi_i \geq 1-y_i(x_i^Tw-b), ~~ \xi_i\geq 0,~\forall i \in T \\
    & \text{\textit{Stationary Constraints}} \\
    & \alpha_i \leq \dfrac{1}{|V_f|},~\forall i \in V_f \\
    & - \sum_{i\in V_f} \alpha_i \bar{y}_i x_{i,j} - \sum_{i\in T} \beta_i y_i x_{i,j} + \mu^+_j - \mu^-_j = 0,~\forall j=1,\dots,D \\
    & \sum_{i\in V_f} \alpha_i \bar{y}_i + \sum_{i\in T} \beta_i y_i = 0 \\
    & \dfrac{\lambda}{|T|} \geq \beta_i,~\forall i \in T \\
        &\text{\textit{Complementary Slackness Constraints}} \\
    &(\dfrac{1}{|V_f|} - \alpha_i)v_i = 0, ~ \forall i \in V_f \\
    & \alpha_i (v_i + \bar{y}_i(x_i^T w-b)-1) = 0, ~ \forall i \in V_f \\
    & \beta_i (\xi_i + y_i(x_i^T w-b)-1) = 0, ~ \forall i \in T \\
    & (\dfrac{\lambda}{|T|} - \beta_i)\xi_i = 0, ~ \forall i \in T\\
    & \mu^+(\bar{w}-w) = 0 \\
    & \mu^-(\bar{w}+w) = 0 \\
    & \lambda \big(\dfrac{1}{|T|} \sum_{i\in T} ((1+\varepsilon)\hat{\xi}_i - \xi_i)\big) = 0\\
    & \alpha_i \geq 0,~\forall i \in V_f, ~ \beta_i \geq 0,~\forall i \in T, \\
    & \lambda \geq 0, \mu^+,\mu^- \geq 0 \numberthis,
\end{align*}
\endgroup
where $\alpha, \beta, \lambda, \mu^+$ and $\mu^-$ are dual variables. The complementary slackness equations are then linearized by using the big-M technique. Namely, for a complementary slackness equation $p q = 0, ~ p\geq 0,~ q\geq 0$, we can linearize it by adding a binary decision variable $z$ with the constraints $0\leq p \leq Mz,~0\leq q \leq M(1-z)$. In \eqref{myeq:single-hyper}, we have on the order $O(|V_f|+|T|+ p)$ complementary slackness equations, where $p$ is the feature dimension. By further applying the \textit{big-M} technique, the final single-level problem will have on the order $O(|V_f|+|T|+p)$ binary decision variables and the corresponding nonlinear constraints. With that, the problem can be solved with existing solvers such as {\tt CPLEX}.

\section{Results \& Discussions}


We have tested our models under two scenarios: limited and perturbed/shifted data. In Section~\ref{UCI}, we present experiments on two real-world classification data sets, Wisconsin Breast Cancer data set (`Cancer') and PIMA Indians Diabetes data set (`Diabetes'), from the UCI Repository for Machine Learning \citep{Dua:2019}. Here, the training and validation sets have limited number of samples. We also present the results when test data are adversarially perturbed. In Section~\ref{TF}, we further test the scenario in which the test data are different from the training data. Specifically, we use a Convolutional Neural Network (CNN) model previously trained on MNIST as a feature extractor and compare the performances of the optimistic and pessimistic formulations as classification models on the FashionMNIST data set. All the implementations are on {\tt Python 3.7.4}. with the optimization problems modeled by {\tt AMPL Python API 2.0} and solved by {\tt CPLEX}. More implementation details of our experiments are provided in Appendix \ref{appdx:implementationDetails}.

\begin{figure*}[b!]
    \centering
    \subfigure[Cancer]{\label{fig:ratio-BC}\includegraphics[width=0.495\columnwidth]{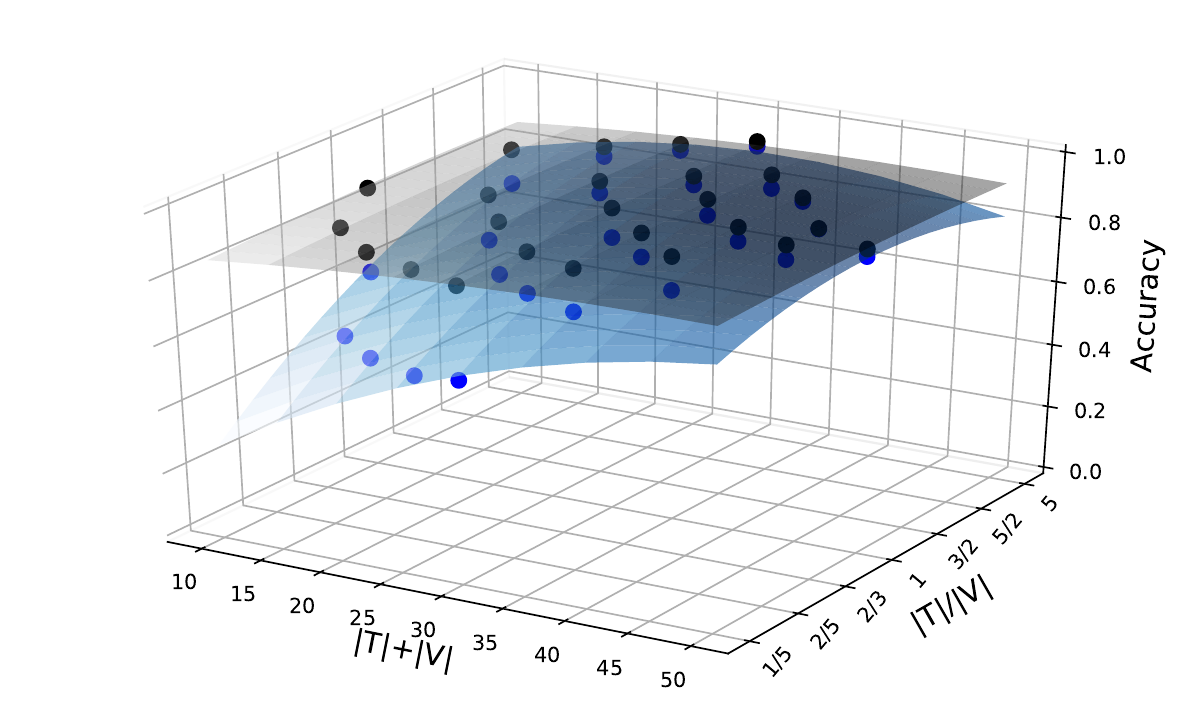}}
    \subfigure[Diabetes]{\label{fig:ratio-PIMA}\includegraphics[width=0.495\columnwidth]{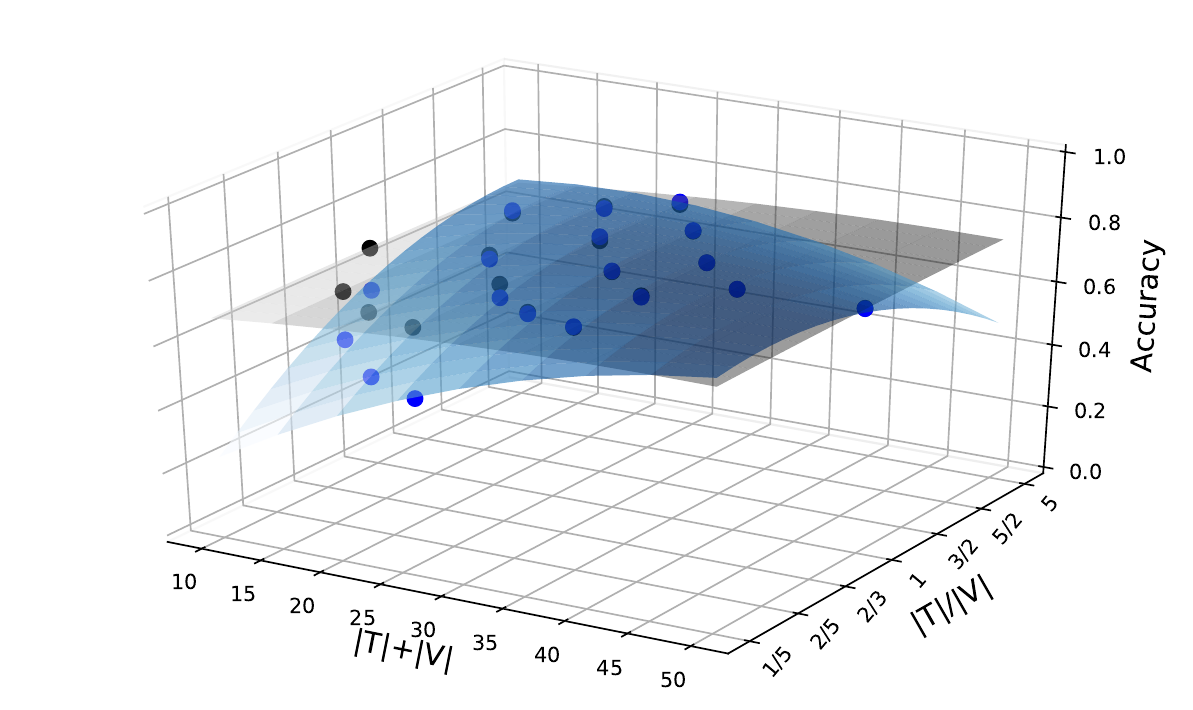}}\vspace{-3.5mm}
    \caption{Comparison of pessimistic and optimistic solutions by the average accuracy with respect to varying total sample size ($|V|+|T|$) and training to validation splitting ratio ($|T|/|V|$); \textbf{Blue circles:} performance results of optimistic models, \textbf{Black circles:} performance results of pessimistic models.}
    \label{fig:3D-splitRatio}\vspace{-3.5mm}
\end{figure*}

\subsection{Experiments with UCI Data Sets}\label{UCI}

As a pre-processing step, duplicated entries are removed from the data sets in our experiments. All the data are normalized to have 0 mean and unit variance. We have conducted the experiments for multiple runs with random training-validation-test splits, and the average performances are reported. Also, the training-validation-test splits are consistent for the corresponding experiments under optimistic and pessimistic cases for fair comparison. We hold out the half of each data set as the test set in all experiments, so that both the optimistic and the pessimistic bilevel hyperparameter optimization methods do not see the held-out test samples. We then sample the remaining samples in each data set as the corresponding training and validation sets. During the training-validation-test splits, the label distribution of the data set is preserved. All the experimental runs are based on random sampling of training, validation and test sets. Table \ref{tab:datasetsAll} summarizes the data set characteristics. 
\begin{table}[b!]
\caption{Summary of data characteristics.}
\begin{center}
\begin{small}
\begin{sc}
    \begin{tabular}{c c c c }
    \toprule
    Data set & 
    Total Sample Size & $|\text{Test Set}|$ & Feature dimension \\ \hline 
    Diabetes & 768  & 384 & 8  \\
    Breast Cancer & 449 & 225 & 9 \\
    \bottomrule
    \end{tabular}
\end{sc}
\end{small}
\end{center}
\label{tab:datasetsAll}
\end{table}


\subsubsection{Experiments with Limited Data}\label{smallData}
We first conduct experiments to see the effect of two formulations of bilevel hyperparameter tuning,  optimistic and pessimistic views, when we have  limited amount of training data. As we have discussed previously, under the small-sample scenario, we argue that the optimistic selection of the hyperparameters at the outer level would overly trust the inner-level solution, which may overfit the limited data. In contrast, the pessimistic formulation, by hedging against uncertainty introduced by the limited number of the training data, may obtain more robust solutions, which better capture the underlying distribution of the data and thus perform better on unseen data.

In Figure \ref{fig:3D-splitRatio}, we present the average test set accuracy results on multiple runs with random training-validation-test splits for varying total number $|V|+|T|$ of validation and training data and split ratio of $|T|/|V|$. Here, the the blue and black circles represent the average testing accuracy from ten random runs with a specific total sample size and split ratio for optimistic and pessimistic models, respectively. To better visualize the performance trends, quadratic polynomials are fitted to the corresponding results for comparison based on the blue and gray fitted surfaces. 

For both Cancer and Diabetes data sets, pessimistic solutions perform consistently better than the optimistic ones, providing higher average testing accuracy. Along both the total sample size and split ratio axes, pessimistic solutions demonstrate significantly better stability thanks to better interactions between outer- and inner-level DMs to handle potential model uncertainty. 
It is also clear that when the total sample size, $|V|+|T|$, is larger than 50, the derived optimistic and pessimistic solutions converge in the conducted experiments as both formulations have the same unique inner-level solution. 
From both plots, with the fixed split ratio, for example $|T|/|V|=1$, the performances of pessimistic solutions decrease much slower as the size of the training set decreases than the optimistic ones. Such a trend indicates that the pessimistic solutions better handle the non-uniqueness of inner-level solutions, especially when $|T|$ is small. Therefore, models achieved by pessimistic bilevel optimization can avoid potential overfitting and are more data-efficient in this small data regime.  

Another noticeable observation from these experiments is that the models obtained by solving pessimistic bilevel optimization is much less sensitive to the training and validation split ($|T|/|V|$), again due to better outer- and inner-level interactions. The optimistic solution depends heavily on the solution quality of the inner-level training model, which naturally requires a larger number of training data to derive generalizable solutions. This can be observed clearly based on the blue circles in the back left corners of the blue surfaces, indicating the performance of the optimistic model degraded significantly when we have a very small training set. For example, in Figure \ref{fig:ratio-BC}, when $|V|+|T|=30$, optimistic solutions have significantly lower average testing accuracy with $|T|/|V|=1/5$; however, the performance improves significantly when we assign more data to the training set when $|T|/|V| \geq 1/2$. Similar behavior is also observed in the Diabetes data set in Figure \ref{fig:ratio-PIMA}. In contrast, pessimistic hyperparameter tuning takes better advantage of validation data and can achieve more stable and better prediction performance with different split ratios with small total sample size, $|V|+|T|$.

\begin{figure*}[ht]
\centering     
    \subfigure{\label{fig:small-BC-val5}\includegraphics[width=0.49\columnwidth]{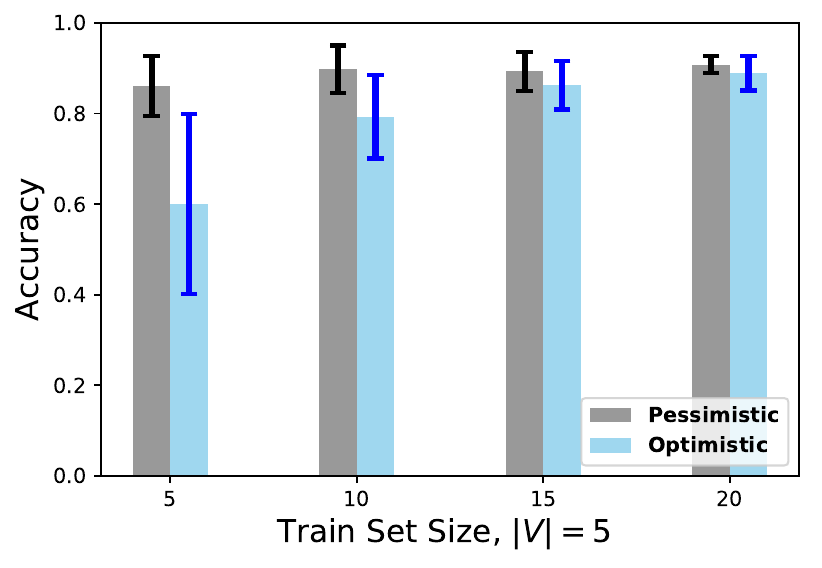}}
    \subfigure{\label{fig:small-BC-val20}\includegraphics[width=0.49\columnwidth]{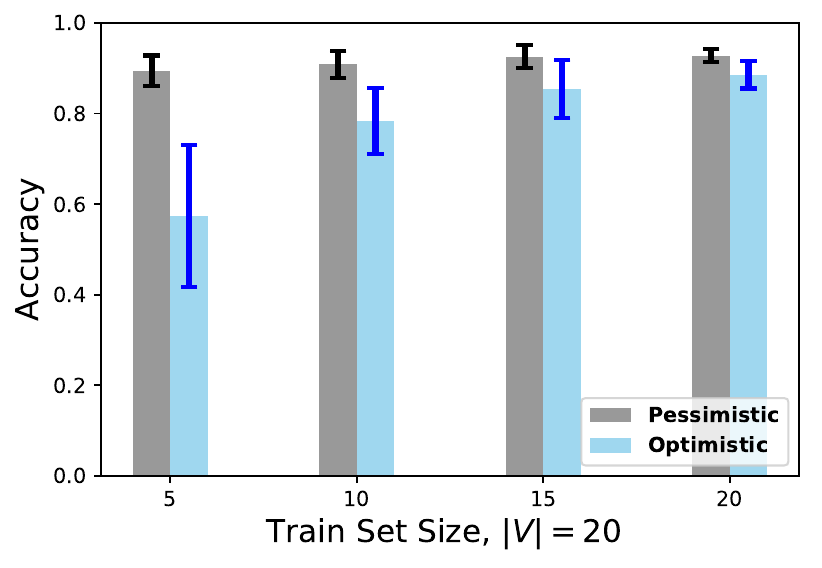}} \\
    \subfigure{\label{fig:small-BC-tr5}\includegraphics[width=0.49\columnwidth]{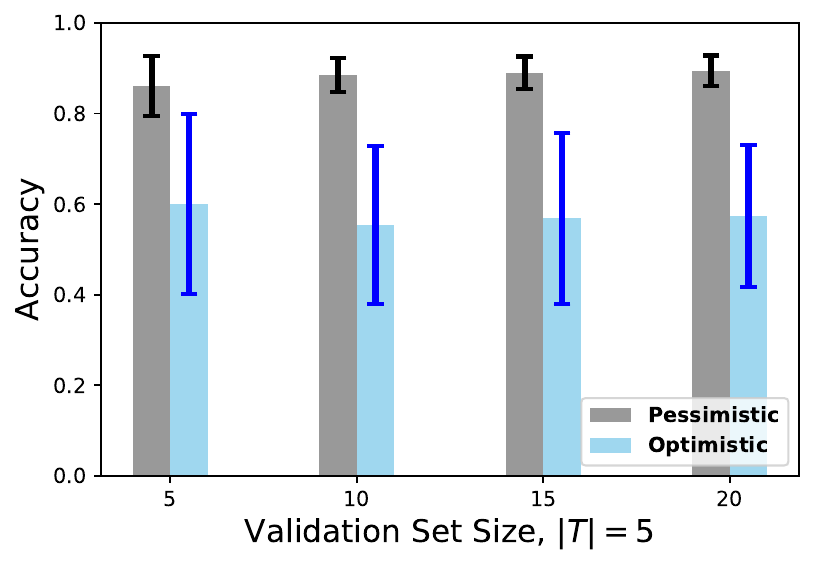}}
    \subfigure{\label{fig:small-BC-tr20}\includegraphics[width=0.49\columnwidth]{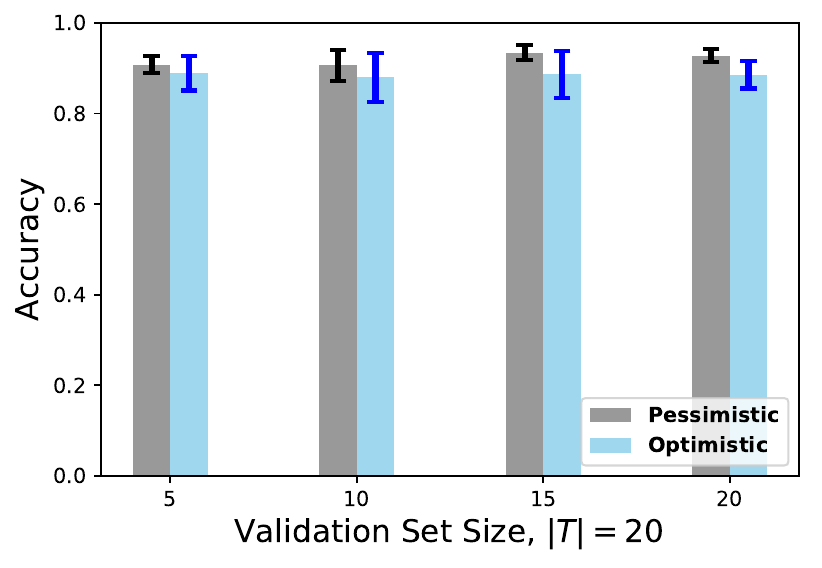}}\vspace{-4mm}
\caption{Average testing accuracy for varying training and validation set sizes for the Cancer data set}
\label{fig:small-accVStrsize}\vspace{-3mm}
\end{figure*}



In Figure \ref{fig:motivEx}, we have plotted the average testing accuracy and the corresponding standard deviation values for varying training set sizes when $|V|=10$ in the Cancer data set. More comprehensive results are provided in Figure \ref{fig:small-accVStrsize} with varying validation and training set sizes. First, increasing the training or validation set size while keeping the other fixed leads to less variability in both pessimistic and optimistic models' performances. From Figures \ref{fig:motivEx} and \ref{fig:small-accVStrsize}, we observe that the optimistic model behaves poorly when $|T|$ is small, and its performance improves as training set size increases for a fixed validation set size. The two plots on the right of Figure \ref{fig:small-accVStrsize} show that the optimistic models rely on good inner-level solutions, and its performance depends heavily on the number of training data. In Figure \ref{fig:small-accVStrsize}, increase in validation data itself is not enough for optimistic models to better generalize to unseen test data. Additional results on the Diabetes data set can be found in Appendix~\ref{appdx:experimentsMNIST}, which show similar trends.

In summary, our empirical studies support that when we have limited training data, due to the inner-level model uncertainty and a higher risk of overfitting, we need more robust hyperparameter solutions to have better performance on the unseen test data. By taking a pessimistic view to protect against unexpected inner-level model behavior, we derive more robust and generalizable models in terms of testing accuracy. More importantly, we observe that the pessimistic model is less sensitive to training and validation splits. Training data are generally overused to tune the hyperparameters for a better generalization performance, i.e., in cross-validation. We believe that our pessimistic model's stable behavior to validation/training splits  motivates new research for training robust model  that is less affected by random splits, which are often done manually as a mysterious art.

\subsubsection{Experimetns with Perturbed Data}\label{perturbedData}
To further investigate the generalizability of the pessimistic and optimistic bilevel hyperparameter tuning, we perform experiments under different scenarios in which the test data are different from training data. In particular, we adversarially perturb the test data by using the {\tt SecMl} Python package \citep{melis2019secml} to create systematic attacks. We test two scenarios to simulate the setups corresponding to adversarial learning and meta learning (or few-shot learning) respectively: in the first setup, we keep the validation set unperturbed so that both model training and hyperparameter tuning do not see the test data; in the second setup, we perturb the validation set to allow hyperparameter tuning to have perturbed data too (the perturbations to the validation and test data can be different).  

We note here that our main interest is not on creating adversarial attacks, and we have used {\tt SecMl} simply as a way to perturb the data at hand. In {\tt SecMl}, adversarial evasion attacks are created to manipulate the samples in the test and/or validation sets. To generate those attacks, we need a trained model, for which we take a classical soft SVM trained using all the training data (the validation and test samples are not used here). At the back end, the attacks are generated by the gradient-based maximum-confidence algorithm \citep{biggio2013evasion}. The perturbations $\Delta$ are constrained to be bounded by $\rho$ in $l_2$-norm, $||\Delta||\leq \rho$. In our experiments, the attacked samples are chosen to be the ones in the margin of the classifying hyperplane obtained by solving the optimistic bilevel hyperparameter optimization problem in \eqref{prob:optim-hinge}. We conduct the experiments to observe how the pessimistic model's performance compares to that of the optimistic counterpart when the test set is different from the training data. Such difference is captured by perturbing the test set with different upper bounds, $\rho$'s, on both UCI data sets. In all the experiments, we only attack the test data that are inside the margin regions of the corresponding hyperplanes derived by solving the optimistic bilevel optimization. 


The average testing accuracy of the derived models by pessimistic formulation and the evaluation of the models derived by the optimistic one, detailed in Appendix~\ref{appdx:optimizationQuality}, are compared for various perturbation levels, $\rho$, and suboptimality rates, $\varepsilon$, in Figures \ref{fig:3Dbar-BC-testDmax_ebs} and \ref{fig:3Dbar-PIMA-testDmax_ebs}.  
In these experiments, $\rho=0$ corresponds to the cases when there is no perturbation to the data sets. The reported average results in this section are computed for five runs with random training-validation-test splits.

\begin{figure}[]
\centering
	\subfigure[Cancer]{\label{fig:cleanVal-BC}\includegraphics[width=0.493\columnwidth]{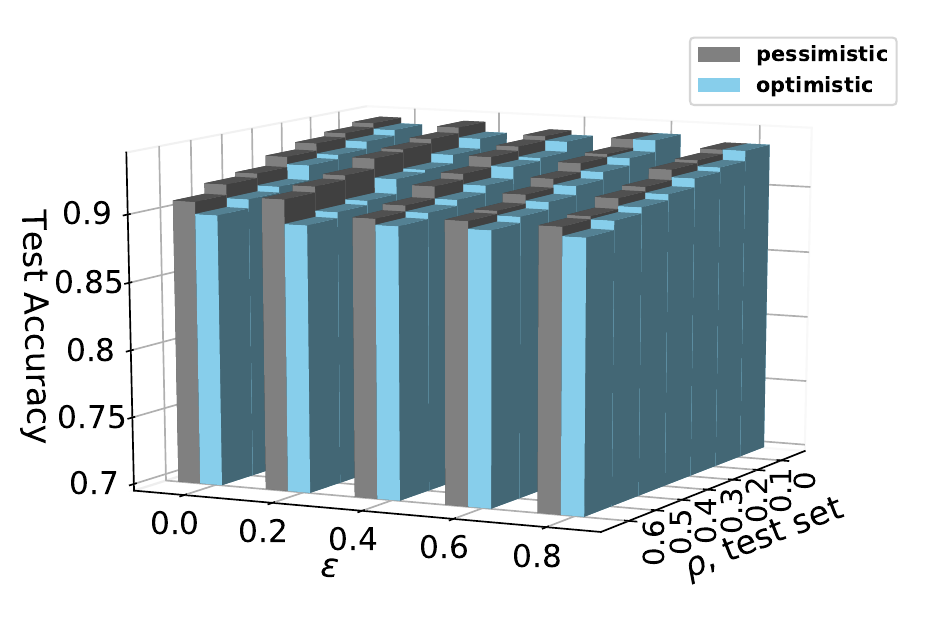}}
	\subfigure[Diabetes]{\label{fig:cleanVal-PIMA}\includegraphics[width=0.493\columnwidth]{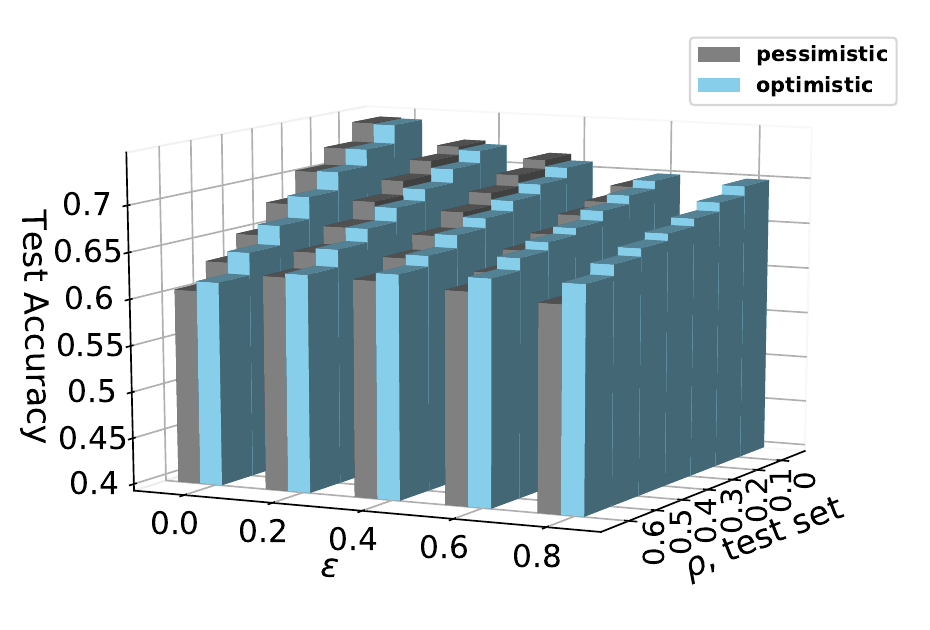}}\vspace{-5mm}
\caption{Testing accuracy comparison for various values of $\varepsilon$ and perturbation bound $\rho$ on $||\Delta||$ with clean validation data.}
\label{fig:3Dbar-BC-testDmax_ebs}\vspace{-5mm}
\end{figure}

Figures \ref{fig:3Dbar-BC-testDmax_ebs} and \ref{fig:3Dbar-PIMA-testDmax_ebs} present the resulting average testing accuracy for the Cancer and Diabetes data sets when the validation set is kept clean or perturbed with $\rho=0.3$ ($||\Delta||\leq 0.3$). 
In both figures, the testing accuracy of optimistic and pessimistic models are very close when $\varepsilon=0$, indicating we have obtained the optimal training models except Figure \ref{fig:adverVal-BC}. As expected, the testing accuracy of both optimistic and pessimistic solutions decreases with as the perturbation levels to test data increase, especially in the Diabetes data set (Figures \ref{fig:cleanVal-PIMA} and \ref{fig:adverVal-PIMA}). 

In Figure \ref{fig:cleanVal-BC}, superiority of pessimistic models manifests itself with $\varepsilon \leq 0.4$ when testing on the perturbed test data in the Cancer data set. The performance difference between optimistic and pessimistic models becomes smaller with large $\varepsilon$. For the Diabetes data in Figure \ref{fig:cleanVal-PIMA}, only the pessimistic solutions with $\varepsilon \leq 0.4$ perform better when the test data have small perturbations ($\rho\leq 0.2$). This setup emulates the adversarial learning scenarios. These results indicate that adverserial learning is challenging, and we may need to impose appropriate inner-level model uncertainty to achieve robust pessimistic solutions by adjusting $\varepsilon$ based on the difference between training and test data. Enforcing large $\varepsilon$ values may blur signals from training data and lead to worse models. 

In Figure \ref{fig:adverVal-BC}, we observe pessimistic models significantly outperform optimistic ones when $\varepsilon=0$ with attacked test data, indicating that pessimistic solutions can take better advantage of the validation set that is similarly perturbed as the test data. More importantly, for both data sets in Figures \ref{fig:adverVal-BC} and \ref{fig:adverVal-PIMA}, pessimistic solutions perform consistently better than optimistic ones when validation data are also perturbed, especially when a suitable value for $\varepsilon$ can be chosen. This demonstrates the promising potential of pessimistic hyperparameter optimization in few-shot learning, online learning, or meta learning when a limited number of new data can serve as the validation set to help fine tune the learned models.

\begin{figure}[h!]
\centering
    \subfigure[Cancer]{\label{fig:adverVal-BC}\includegraphics[width=0.493\columnwidth]{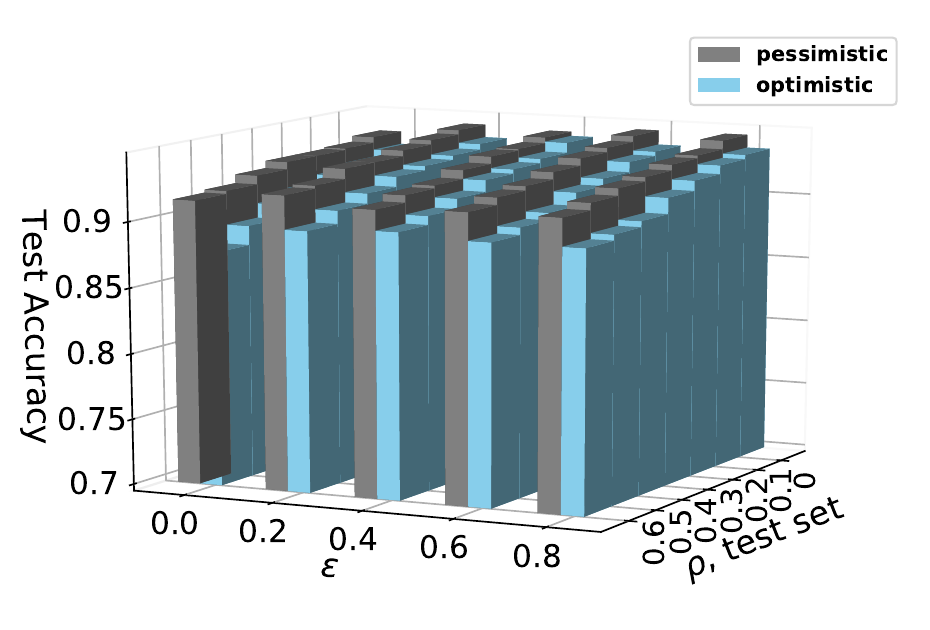}}
    \subfigure[Diabetes]{\label{fig:adverVal-PIMA}\includegraphics[width=0.493\columnwidth]{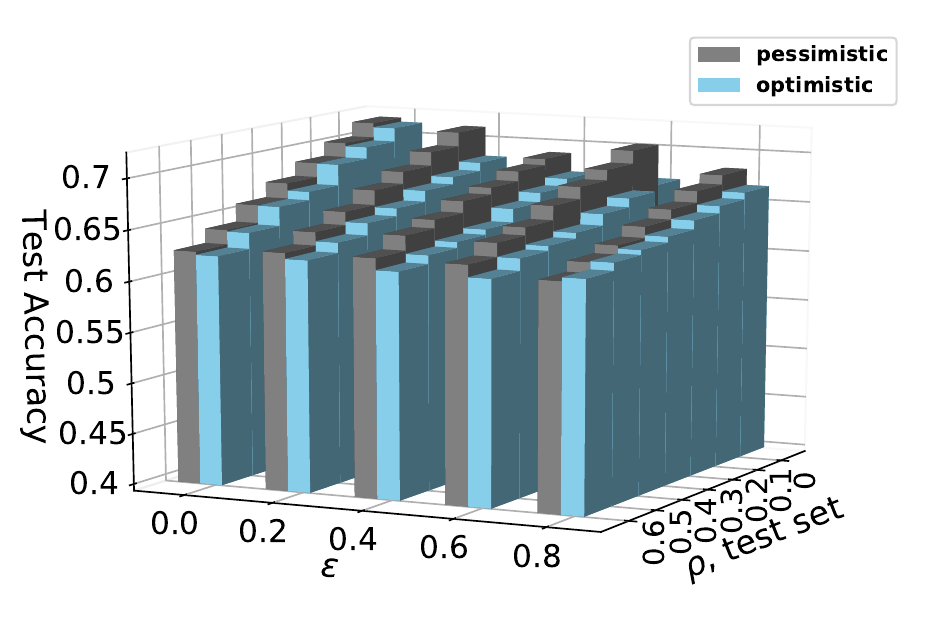}}\vspace{-5mm}
    \caption{Testing accuracy comparison for various values of $\varepsilon$ and perturbation bound $\rho$ on $||\Delta||$ with perturbed validation data.}
    \label{fig:3Dbar-PIMA-testDmax_ebs}\vspace{-5mm}
\end{figure}

\subsection{Experiments with MNIST \& FashionMNIST}\label{TF}
To further investigate the generalizability of the pessimistic and optimistic bilevel hyperparameter tuning, we perform experiments on transfer learning, in which test data are different from training data. In particular, we first train a CNN model on the MNIST benchmark. The model has three convolution layers of $3\times3$ filters with the ReLU activation. The first and last convolution layers are followed by a max pooling layer of $2\times2$ filters and then a fully connected layer with the dropout rate at $0.25$. The output size of the last layer before the classification layer is $100$. The last layer has the softmax activation function with $10$ outputs for the corresponding digits. The model is optimized by the Adam optimizer with a learning rate of $0.0001$ for 15 epochs. While the model achieves an accuracy of 0.989 on MNIST, this rate drops significantly to 0.875 on FashionMNIST when the exactly same model is used. 

We use the aforementioned CNN model without the last classification layer as a feature extractor with $100$ features. We randomly sample $50$ examples for training and $50$ examples for validation sets from the images of the digits `1' and `7' uniformly. We have conducted the experiments for five runs with random training-validation splits, and the testing accuracy on FashionMNIST is evaluated for optimistic and pessimistic hyperparameter tuning. The performance of the pessimistic models is compared to that of the evaluated optimistic models under the worst-case scenario for a given suboptimality rate, $\varepsilon$. Appendix~\ref{appdx:optimizationQuality} details this procedure. In short, with larger $\varepsilon$ values, suboptimal but neighboring models obtained by the training problem are considered, and we investigated how the optimistic models would have performed under the worst-case scenario. Hence, $\varepsilon=0$ does not allow such suboptimality, and shows comparison of the performance of both models. The average accuracy results are presented at various levels of suboptimality $\varepsilon$ in Figure~\ref{fig:TF-fashion}. Incorporating such $\varepsilon$-suboptimality would be valuable if the training and test set are from different data distributions or in face of adversarial scenarios.

\begin{figure}[h]
\begin{center}
\centerline{\includegraphics[width=0.60\columnwidth]{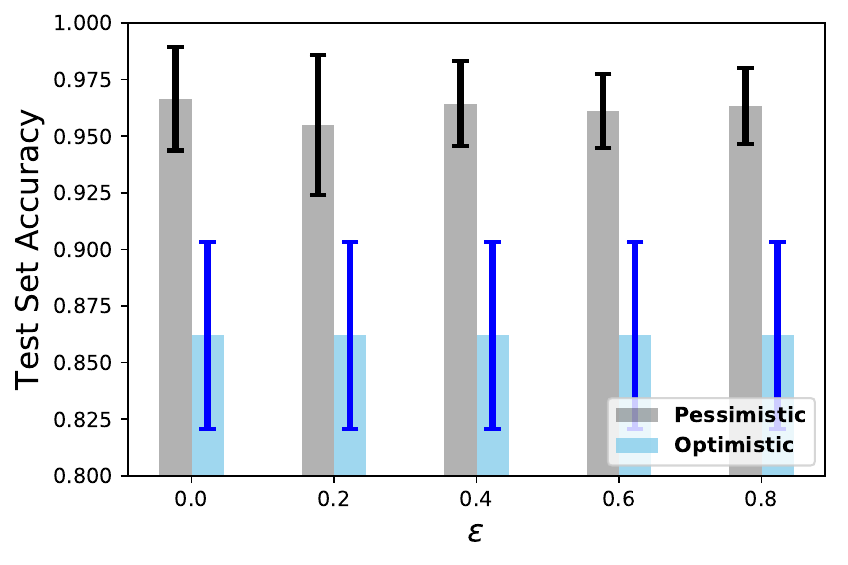}}\vspace{-3mm}\vspace{-2mm}
\caption{Comparison of pessimistic and optimistic solutions by the average testing accuracy of five runs with random splits between training and validation sets.}
\label{fig:TF-fashion}\vspace{-3mm}
\end{center}
\vskip -0.2in
\end{figure}
Clearly, there manifest significant differences between the performances of the optimistic and pessimistic models. While the pessimistic formulation achieves an average accuracy of $0.97$ by multiple trials on FashionMNIST, the optimistic model has an average accuracy around $0.86$ with a higher variation across $\varepsilon$'s. 

In \ref{appdx:experimentsMNIST}, we report some example of FashionMNIST test images that are classified correctly by the pessimistic model and misclassifed by the optimistic model.

In summary, considering different machine learning tasks when test data may be different from training data, we can benefit from our pessimistic bilevel hyperparameter optimization when the uncertainty of the learned model is appropriately incorporated. It is clear that, compared to commonly adopted optimistic solutions in the current AutoML practice, it can take better advantage of additional validation data if they are similar as test data, again thanks to better feedback from  the inner-level training model to the outer-level hyperparameter selection. 

\section{Conclusions \& Future Research}

In this paper, we have proposed the first \textbf{PBL} formulation for hyperparameter tuning and have shown that it is necessary to consider this pessimistic view, especially when learning complex models with limited data or perturbed test data from training. We have analyzed and demonstrated why commonly adopted optimistic hyperparameter tuning can be problematic. To solve challenging three-level \textbf{PBL} problem for hyperparameter tuning in binary classification, a computationally practical procedure was developed based on strong relaxation and approximation. By conducting extensive experiments on two UCI data sets, we have observed that pessimistic solutions to the bilevel hyperparameter tuning problems are robust and have better prediction performances on unseen data. Such robustness was particularly observed in the case where we can only access limited amount of training data. In such case, the inner-level model involves more uncertainty and has a high risk of overfitting. It is also observed that our pessimistic solution can take better advantage of available data than the optimistic counterpart in terms of being less sensitive to the split of training and validation sets thanks to feedback interactions between the outer- and inner-level DMs. Moreover, from our observations on perturbed test data, pessimistic models can achieve better performance when the unseen test data follow different or perturbed distributions from that of the training data. 

Future research directions involve developing more general, scalable and efficient \textbf{PBL} solutions, in particular considering the scenarios when the inner-level global optimality is difficult to guarantee. With that, \textbf{PBL} can become practical tools in AutoML, especially in few-shot learning and online learning where challenges due to small sample size may arise; and in meta learning and adversarial learning problems where the test data are different from the training data.

\acks{The presented work was supported in part by the National Science Foundation (NSF) Awards 1553281, 2212419 and 2215573, as well as the U.S. Department of Energy, Office of Science, Office of Advanced Scientific Computing Research, Mathematical Multifaceted Integrated Capability Centers program under Award B\&R\# KJ0401010/FWP\# CC130. }





\newpage
\appendix
\onecolumn

\newcounter{defcounter}
\setcounter{defcounter}{0}
\newenvironment{myequation}{%
\addtocounter{equation}{-1}
\refstepcounter{defcounter}
\renewcommand\theequation{A\thedefcounter}
\begin{equation}}
{\end{equation}}





\section{Implementation Details of Our Experiments}\label{appdx:implementationDetails}

All the implementations on {\tt Python 3.7.4}. with the optimization problems solved by {\tt CPLEX}. All the reported experiments are based on the implementations on a personal computer with an Intel Core i7 processor at 2.9GHz, running {\tt macOS Mojave version 10.14.6}. In all of the experiments with UCI data sets, we have set $\bar{w}^{UB}=1$ and $\bar{w}^{LB}=0$ for the box constraints on hyperparameters, and $\bar{w}^{UB}=2.5$ and $\bar{w}^{LB}=0$ for MNIST and FashionMNIST. 

In the toy experiments with UCI data sets (Section~\ref{smallData} in the main text), ten random runs are performed for the random training-validation-test splits. The training and validation sizes are self-expressing in the results of this section. During these experiments, for the Cancer data set, we have flipped the class labels of all the points in the randomly sampled validation set for each run. For the Diabetes data set, we have flipped the labels of the points in the margin of the optimistic hyperplane as well as the misclassified ones when the training size is larger than $20$ and the validation size is larger than $15$. For the other configurations, we have flipped the labels of all the points in the validation set. 

\section{Additional Experimental Results}

\subsection{Optimization Solution Quality}\label{appdx:optimizationQuality}

\renewcommand{\thefigure}{A\arabic{figure}}

\setcounter{figure}{0}

\begin{figure}[b!]
\begin{center}   
    \subfigure[Cancer]{\label{fig:evalu-BC}\includegraphics[width=0.45\columnwidth]{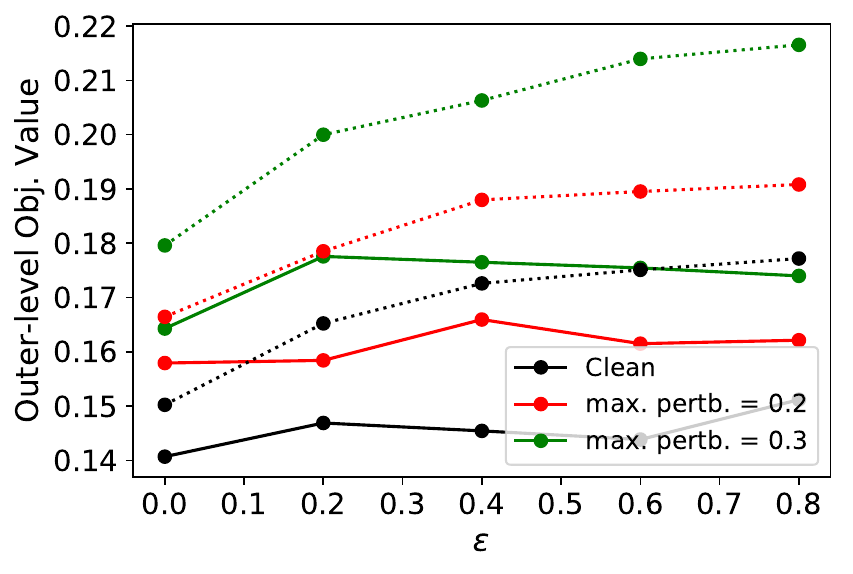}}
    \subfigure[Diabetes]{\label{fig:evalu-PIMA}\includegraphics[width=0.45\columnwidth]{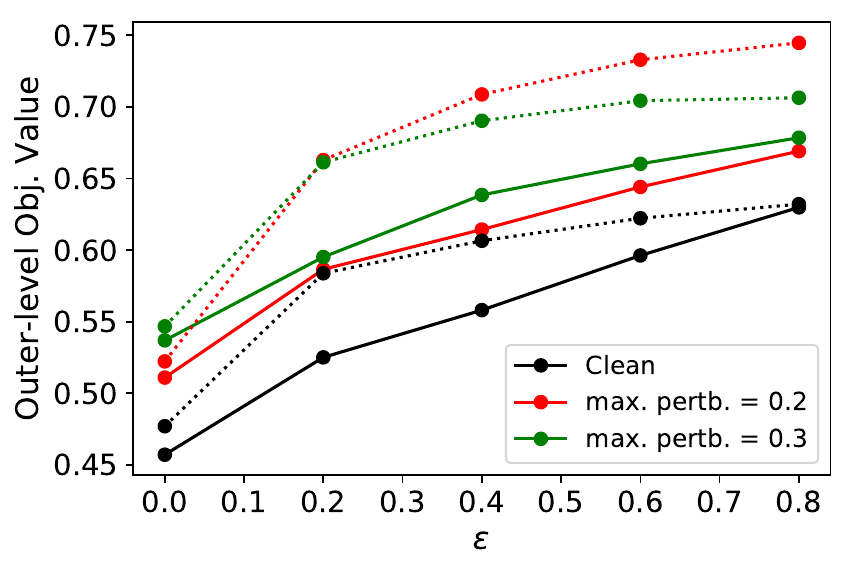}} \vspace{-3mm}
\end{center}    
\caption {Evaluation experiments for validation of the optimization solution quality. \textbf{Dashed line:} optimistic models, \textbf{solid line:} pessimistic models.}
\label{fig:evalu}
\end{figure}

As we have used the flipping operation to approximate the inner-level non-concave maximization problem in \eqref{RPBL-hyper} of the main text, we have performed an experiment to evaluate the optimization solution quality. To do so, we compare the outer-level objective function values of the pessimistic, \eqref{A-RPBL-hyper} in the main text, and optimistic models under a worst-case scenario. To obtain the evaluated optimistic model under a worst-case scenario, namely the model parameters $w$ and $b$, we first solve the optimistic bilevel hyperparameter optimization problem in \eqref{prob:optim-hinge}. Then, we compute an optimal outer-level solution $\bar{w}^*$ along with the corresponding inner-level optimal objective value $\theta(\bar{w}^*)$ given $\bar{w}^*$. We then solve the following optimization problem for various values of $\varepsilon$:

\setcounter{defcounter}{0}
\begin{myequation}
    \begin{aligned}
    \label{evaluation-ML}
    \underset{w,b,\xi}{\text{max.}} ~ & \dfrac{1}{|V|}\sum_{i\in V} [1-y_i(x_i^Tw-b)]_+ \\
    \text{s.t.} ~ & -\bar{w}^* \leq w \leq \bar{w}^* \\
    & \dfrac{1}{|T|}\sum_{i\in T}{\xi_i} \leq \theta(\bar{w}^*)(1+\varepsilon) \\
    & \xi_i \geq 1-y_i(x_i^Tw-b),~\forall i \in T \\
    & \xi_i \geq 0,~\forall i \in T.
    \end{aligned}
\end{myequation}
To solve the problem in \eqref{evaluation-ML}, we apply the previously mentioned flipping operation to compute the pessimistic solution and solve the minimization problem over the validation examples with flipped labels. Such comparison of the objective function values provides us a way to evaluate the performance of our approximate procedure as the objective function value of \eqref{evaluation-ML} should be consistently larger than that of the pessimistic one for various $\varepsilon$ values due to the nature of these two formulations. 
Figure~\ref{fig:evalu} presents the corresponding average outer-level objective function values of the optimistic and pessimistic formulations for five runs with random training-validation-test splits. The performances of the optimistic and pessimistic models are shown in dashed and solid lines, respectively. We have also computed the outer-level objective function values for different levels of perturbation ($\rho$) in the validation sets as described in Section~\ref{perturbedData}. These results are color coded in green, red and black. We observe that, for optimistic solutions, the outer-level objective function values monotonically increase in $\varepsilon$, i.e., it increase as the inner-level solution sets become larger. This demonstrates that optimistic solutions can have degraded prediction performance with higher hinge loss on the validation set when the uniqueness assumption of the inner-level solution is violated. At the same time, for pessimistic solutions, the outer-level objective function values stay approximately same (see Fig.~\ref{fig:evalu}(a)) or increase in $\varepsilon$ (see Fig.~\ref{fig:evalu}(b)) but still remaining smaller than the objective function values of the optimistic counterparts. Also, note the slower increase of the outer-level objective function values for the pessimistic solutions from $\varepsilon=0$ to $0.2$, compared to those of the optimistic solutions in Fig.~\ref{fig:evalu}(b). More critically, the outer-level objective function values on the validation set obtained by the optimistic models under a worst-case scenario are consistently larger than the outer-level objective values of the approximated pessimistic counterparts. All these behavior trends of the outer-level objective functions values validate the solution quality of our pessimistic bilevel hyperparameter optimization formulation.


\subsection{Experiments: UCI Data sets}\label{appdx:experimentsUCI}

In Figure~\ref{fig:small-accVStrsize-diabet}, we present additional experimental results with the Diabetes data set with limited numbers of training and validation data. It is clear that the behavior trends on the Diabetes data set are consistent with the ones on the Cancer data set as described in Section~\ref{smallData} of the main text, further confirming the necessity of pessimistic bilevel hyperparameter optimization under data scarcity. 

\begin{figure}[ht!]
\begin{center}
    \subfigure{\label{fig:small-PIMA-val5}\includegraphics[width=0.49\columnwidth]{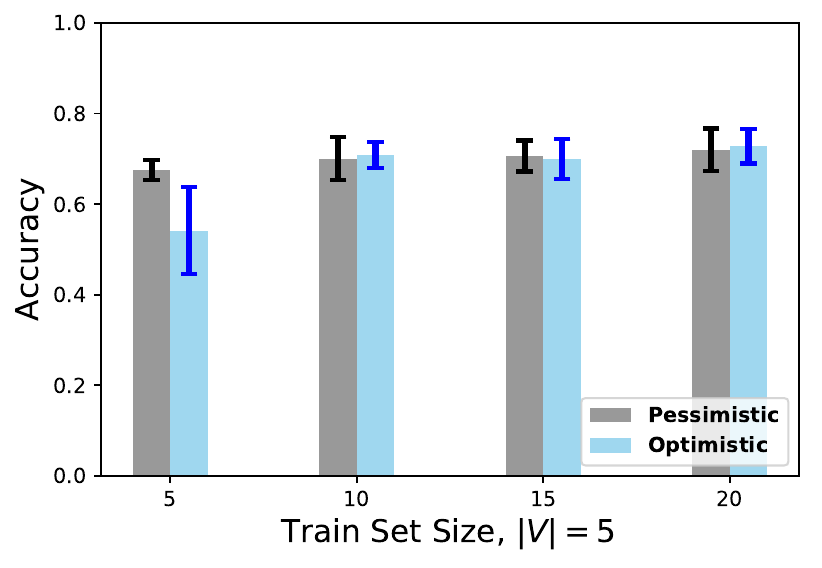}}
    \subfigure{\label{fig:small-PIMA-val20}\includegraphics[width=0.49\columnwidth]{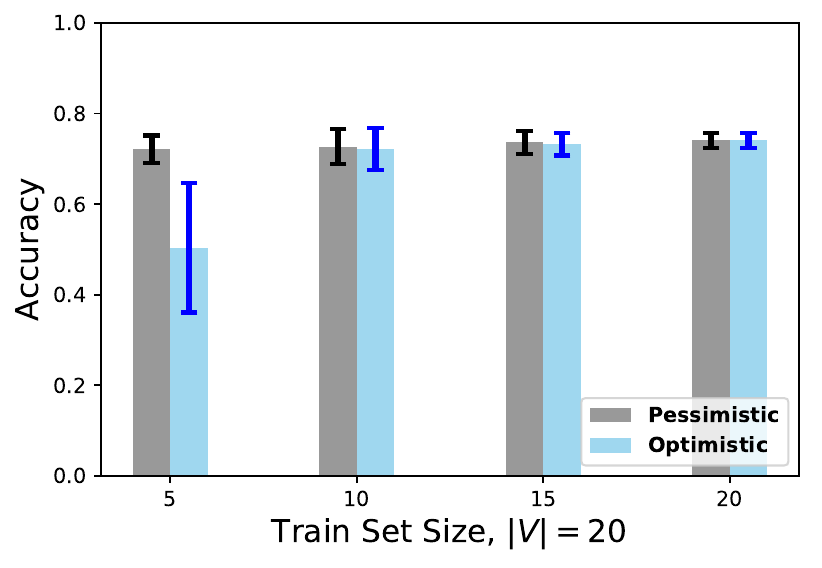}} \\
    \subfigure{\label{fig:small-PIMA-tr5}\includegraphics[width=0.49\columnwidth]{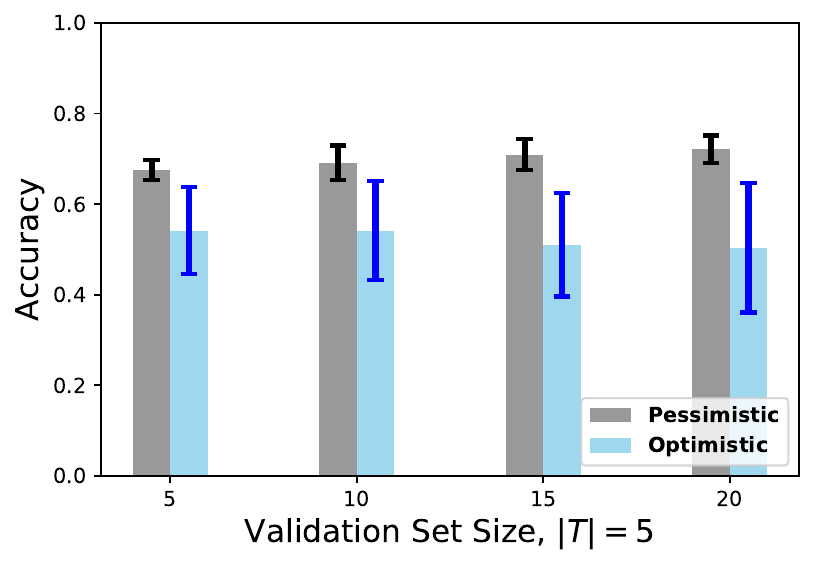}}
    \subfigure{\label{fig:small-PIMA-tr20}\includegraphics[width=0.49\columnwidth]{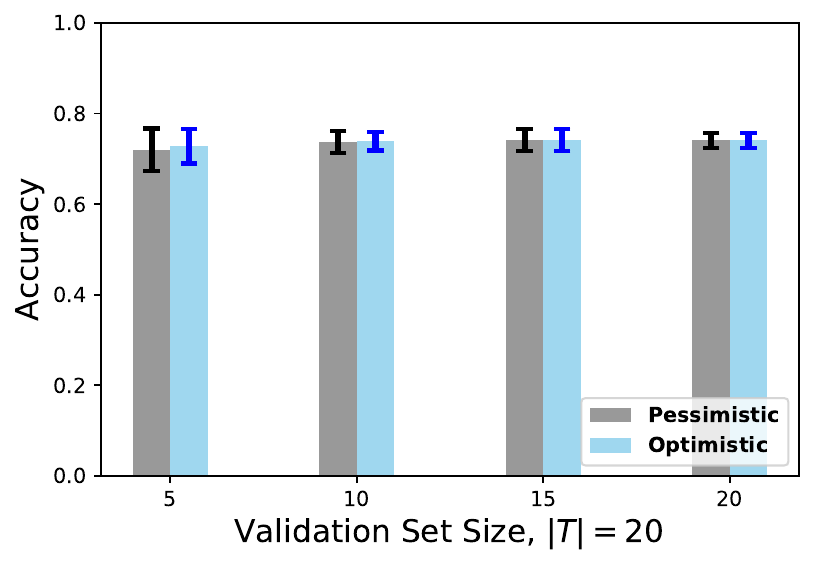}} \vspace{-5mm}
\end{center}    
\caption{Average testing accuracy for varying training and validation set sizes for the Diabetes data set.}
\label{fig:small-accVStrsize-diabet}
\end{figure}

In addition to these plots, we provide the quantitative values of the averaging testing accuracy~(acc) and the corresponding standard deviation~(std) in Table~\ref{tab:small_PIMA-BC}. We also provide the corresponding entries for the additional experiments with $(|V|=10,~|T|=10)$, $(|V|=10,~|T|=15)$, $(|V|=15,~|T|=10)$, and $(|V|=15,~|T|=15)$ in the table for the sake of the completeness of the results.

\setcounter{table}{0}
\renewcommand{\thetable}{A\arabic{table}}

\begin{table}[b!]
\small
    \caption{Averaging testing accuracy (acc) and the corresponding standard deviation (std) values of pessimistic and optimistic solutions with different training ($|T|$) and validation set size ($|V|$) on both the Cancer and Diabetes UCI data sets.} 
    \begin{center}
    \begin{subtable}
    \centering
    \begin{tabular}{c c c c}
    \toprule
\multicolumn{4}{c}{Cancer (Figure~\ref{fig:small-accVStrsize})} \\ \toprule    
\multirow{2}{*}{$|V|$} & \multirow{2}{*}{$|T|$} &   \multicolumn{2}{c}{acc $\pm$ std }  \\      
    &   & pessimistic & optimistic  \\ \hline
 5  & 5     & $0.861 \pm 0.066$ & $0.600 \pm 0.199$ \\
 5  & 10    & $0.898 \pm 0.052$ & $0.793 \pm 0.092$ \\
 5  & 15    & $0.893 \pm 0.043$ & $0.863 \pm 0.054$ \\
 5  & 20    & $0.908 \pm 0.019$ & $0.889 \pm 0.038$ \\
 
 10  & 5     & $0.885 \pm 0.038$ & $0.553 \pm 0.174$ \\
 10  & 10    & $0.901 \pm 0.039$ & $0.763 \pm 0.101$ \\
 10  & 15    & $0.906 \pm 0.031$ & $0.870 \pm 0.065$ \\
 10  & 20    & $0.907 \pm 0.034$ & $0.880 \pm 0.054$ \\
 
 15  & 5     & $0.890 \pm 0.036$    & $0.568 \pm 0.189$ \\
 15  & 10    & $0.920 \pm 0.030$    & $0.760 \pm 0.099$ \\
 15  & 15    & $0.923 \pm 0.026$    & $0.834 \pm 0.050$ \\
 15  & 20    & $0.935 \pm 0.016$    & $0.886 \pm 0.053$ \\
 
 20  & 5     & $0.895 \pm 0.033$ & $0.574 \pm 0.157$ \\
 20  & 10    & $0.908 \pm 0.029$ & $0.784 \pm 0.073$ \\
 20  & 15    & $0.926 \pm 0.025$ & $0.854 \pm 0.064$ \\
 20  & 20    & $0.928 \pm 0.015$ & $0.886 \pm 0.031$ \\ 
 \bottomrule
    \end{tabular}
    \label{tab:small_BC}
    \end{subtable}%
    \begin{subtable}
    \centering
    \begin{tabular}{c c c c}
    \toprule
\multicolumn{4}{c}{Diabetes (Figure~\ref{fig:small-accVStrsize-diabet})} \\  \toprule    
\multirow{2}{*}{$|V|$} & \multirow{2}{*}{$|T|$} &   \multicolumn{2}{c}{acc $\pm$ std }  \\      
    &   & pessimistic & optimistic  \\ \hline
 5  & 5     & $0.674 \pm 0.022$ & $0.542 \pm 0.097$ \\
 5  & 10    & $0.700 \pm 0.047$ & $0.708 \pm 0.029$ \\
 5  & 15    & $0.706 \pm 0.034$ & $0.699 \pm 0.044$ \\
 5  & 20    & $0.720 \pm 0.047$ & $0.728 \pm 0.038$ \\
 
 10  & 5     & $0.691 \pm 0.038$ & $0.542 \pm 0.109$ \\
 10  & 10    & $0.717 \pm 0.046$ & $0.704 \pm 0.035$ \\
 10  & 15    & $0.722 \pm 0.031$ & $0.735 \pm 0.031$ \\
 10  & 20    & $0.737 \pm 0.025$ & $0.739 \pm 0.021$ \\
 
 15  & 5     & $0.709 \pm 0.034$  & $0.510 \pm 0.114$ \\
 15  & 10    & $0.731 \pm 0.031$  & $0.689 \pm 0.059$ \\
 15  & 15    & $0.729 \pm 0.034$  & $0.732 \pm 0.033$ \\
 15  & 20    & $0.742 \pm 0.024$  & $0.742 \pm 0.024$ \\
 
 20  & 5     & $0.721 \pm 0.030$ & $0.504 \pm 0.143$ \\
 20  & 10    & $0.727 \pm 0.039$ & $0.722 \pm 0.046$ \\
 20  & 15    & $0.736 \pm 0.025$ & $0.732 \pm 0.025$ \\
 20  & 20    & $0.741 \pm 0.016$ & $0.741 \pm 0.016$ \\ 
 \bottomrule
    \end{tabular}
    \label{tab:small_PIMA}
    \end{subtable}%
\end{center}
\label{tab:small_PIMA-BC}
\end{table}


We also provide the quantitative values of the averaging testing accuracy and the corresponding standard deviation for the perturbed data experiments presented in Section~\ref{perturbedData} of the main text. Due to the large number of different setups we have tested in Figures \ref{fig:3Dbar-BC-testDmax_ebs} and \ref{fig:3Dbar-PIMA-testDmax_ebs} of the main text, here we only provide the corresponding values with a fixed inner-level suboptimality rate $\varepsilon$ ($=0.2$) in Tables~\ref{tab:pert-BC-PIMA} and~\ref{tab:pertVal-BC-PIMA}, respectively. The averaging testing accuracy values given by pessimistic solutions are larger than the ones given by optimistic solutions (except the last row in Table~\ref{tab:pert-BC-PIMA} for the Diabetes data set). This difference is greater for the results in Table~\ref{tab:pertVal-BC-PIMA}, when the perturbation is $0.3$ on the validation set. As discussed in Section~\ref{perturbedData} of the main text, when imposing appropriate inner-level model uncertainty, pessimistic bilevel hyperparameter optimization can achieve more robust predictions on shifted or adversarially perturbed test data than the optimistic counterpart as commonly adopted in the existing AutoML practice.

\begin{table}[b!]
\small
    \caption{Averaging testing accuracy (acc) and the corresponding standard deviation (std) values of pessimistic and optimistic solutions with $\varepsilon=0.2$ for the experiments with different perturbations on test data but no perturbation on validation data.}
\resizebox{\textwidth}{!}{\begin{subtable}
    \centering
    \begin{tabular}{c c c c}
    \toprule
\multicolumn{4}{c}{Cancer (Figure~\ref{fig:cleanVal-BC})} \\  \toprule    
\multirow{2}{*}{$\rho$-validation} & \multirow{2}{*}{$\rho$-test} &   \multicolumn{2}{c}{acc $\pm$ std }  \\ 
  &   & pessimistic & optimistic  \\ \hline
 0  & 0      & $0.938 \pm 0.007$ & $0.930 \pm 0.021$  \\
 0  & 0.1    & $0.933 \pm 0.010$ & $0.920 \pm 0.026$ \\
 0  & 0.2    & $0.932 \pm 0.010$ & $0.914 \pm 0.025$ \\
 0  & 0.3    & $0.929 \pm 0.013$ & $0.911 \pm 0.026$ \\
 0  & 0.4    & $0.921 \pm 0.015$ & $0.904 \pm 0.027$  \\
 0  & 0.5    & $0.917 \pm 0.016$ & $0.899 \pm 0.028$ \\
 0  & 0.6    & $0.911 \pm 0.020$ & $0.895 \pm 0.025$ \\
 \bottomrule
    \end{tabular}
   \end{subtable}%
\begin{subtable}
    \centering
    \begin{tabular}{c c c c}
    \toprule
\multicolumn{4}{c}{Diabetes (Figure~\ref{fig:cleanVal-PIMA})} \\  \toprule    
\multirow{2}{*}{$\rho$-validation} & \multirow{2}{*}{$\rho$-test} &   \multicolumn{2}{c}{acc $\pm$ std }  \\  
  &   & pessimistic & optimistic  \\ \hline
 0  & 0      & $0.727 \pm 0.024$  & $0.723 \pm 0.014$  \\
 0  & 0.1    & $0.718 \pm 0.022$  & $0.710 \pm 0.009$ \\
 0  & 0.2    & $0.703 \pm 0.019$  & $0.695 \pm 0.013$ \\
 0  & 0.3    & $0.687 \pm 0.020$  & $0.682 \pm 0.020$ \\
 0  & 0.4    & $0.667 \pm 0.016$  & $0.667 \pm 0.023$  \\
 0  & 0.5    & $0.647 \pm 0.022$  & $0.652 \pm 0.029$ \\
 0  & 0.6    & $0.629 \pm 0.027$  & $0.633 \pm 0.035$ \\
 \bottomrule
    \end{tabular}
\end{subtable}}%
\label{tab:pert-BC-PIMA}
\end{table}


\begin{table}[!]
\small
    \caption{Averaging testing accuracy and the corresponding standard deviation (std) values of pessimistic and optimistic solutions with $\varepsilon=0.2$ for the experiments with different perturbations on test data and perturbations ($\rho=0.3$) on validation data.}
 \resizebox{\textwidth}{!}{\begin{subtable}
    \centering
    \begin{tabular}{c c c c}
    \toprule
\multicolumn{4}{c}{Cancer (Figure~\ref{fig:adverVal-BC})} \\  \toprule    
\multirow{2}{*}{$\rho$-validation} & \multirow{2}{*}{$\rho$-test} &   \multicolumn{2}{c}{acc $\pm$ std }  \\ 
  &   & pessimistic & optimistic  \\ \hline
 0.3  & 0      & $0.945 \pm 0.080$  & $0.933 \pm 0.013$  \\ 
 0.3  & 0.1    & $0.944 \pm 0.080$  & $0.929 \pm 0.015$ \\
 0.3  & 0.2    & $0.939 \pm 0.090$  & $0.924 \pm 0.018$ \\
 0.3  & 0.3    & $0.935 \pm 0.090$  & $0.921 \pm 0.020$ \\
 0.3  & 0.4    & $0.935 \pm 0.090$  & $0.914 \pm 0.018$  \\
 0.3  & 0.5    & $0.929 \pm 0.013$  & $0.905 \pm 0.024$ \\
 0.3  & 0.6    & $0.926 \pm 0.012$  & $0.892 \pm 0.033$ \\
 \bottomrule
    \end{tabular}
    \label{tab:pert-BC}
    \end{subtable}%
\begin{subtable}
    \centering
    \begin{tabular}{c c c c}
    \toprule
\multicolumn{4}{c}{Diabetes (Figure~\ref{fig:adverVal-PIMA})} \\  \toprule    
\multirow{2}{*}{$\rho$-validation} & \multirow{2}{*}{$\rho$-test} &   \multicolumn{2}{c}{acc $\pm$ std }  \\  
  &   & pessimistic & optimistic  \\ \hline
 0.3  & 0      & $0.713 \pm 0.024$  & $0.683 \pm 0.016$  \\
 0.3  & 0.1    & $0.699 \pm 0.032$  & $0.675 \pm 0.010$ \\
 0.3  & 0.2    & $0.685 \pm 0.036$  & $0.668 \pm 0.014$ \\
 0.3  & 0.3    & $0.673 \pm 0.038$  & $0.657 \pm 0.025$ \\
 0.3  & 0.4    & $0.659 \pm 0.043$  & $0.649 \pm 0.032$  \\
 0.3  & 0.5    & $0.646 \pm 0.049$  & $0.637 \pm 0.044$ \\
 0.3  & 0.6    & $0.633 \pm 0.051$  & $0.627 \pm 0.057$ \\
 \bottomrule
    \end{tabular}
    \end{subtable}}%
    \label{tab:pertVal-BC-PIMA}
\end{table}

\subsection{Experiments: MNIST and FashionMNIST}\label{appdx:experimentsMNIST}
In Figure~\ref{fig:TF-fashion-images}, we present some example images from the test set of FashionMNIST which are missclassified by the optimistic model and correctly classified by the pessimistic model. 

\begin{figure}[h]
\begin{center}
\centerline{\includegraphics[width=0.85\columnwidth]{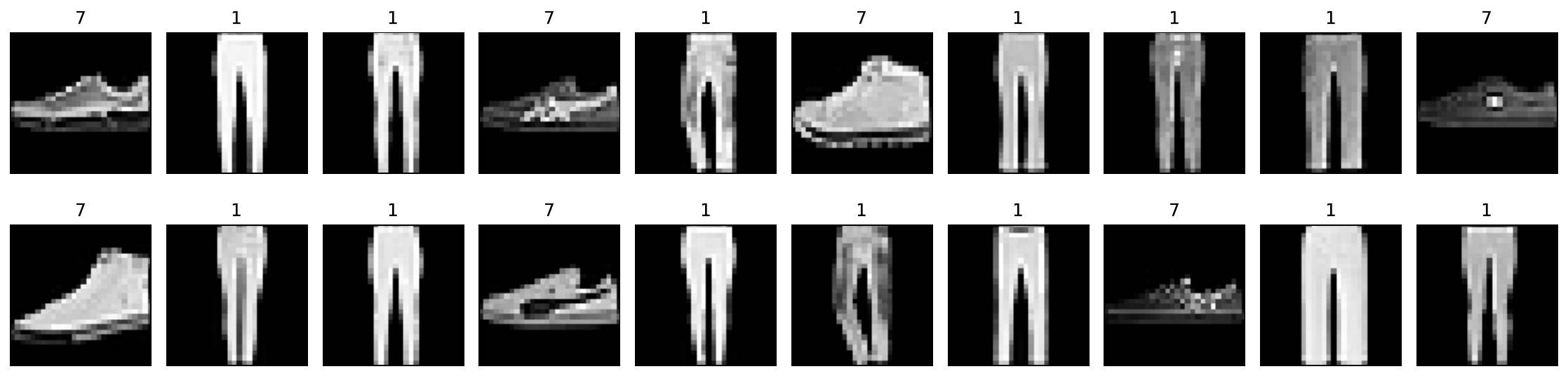}}\vspace{-3mm}\vspace{-2mm}
\caption{Examples of FashionMNIST images which are misclassifed by optimistic model but classified correctly by pessimistic model. True digit labels are given on top of each image.}
\label{fig:TF-fashion-images}\vspace{-3mm}
\end{center}
\vskip -0.2in
\end{figure}

\newpage

\clearpage

\bibliography{references}

\end{document}